\documentclass[journal]{IEEEtran}
\ifCLASSINFOpdf
\else
\fi
\usepackage{times}
\usepackage{graphicx} 
\usepackage{subfigure} 
\usepackage[utf8]{inputenc}
\usepackage{pgfplots}
\usepackage{algorithm}
\usepackage{algorithmic}
\usepackage[shrink=100]{microtype}

\usepackage{bm}
\usepackage{amsmath}
\usepackage{amssymb}
\usepackage{amsthm}
\usepackage{pgfplots}
\usepackage{tikz}
\usepackage{epstopdf}
\pgfplotsset{compat=newest}

\include{new_commands}

\newcommand{\tr}{\mathop{\rm tr}}
\newcommand{\E}{\mathop{\mathbb{E}}}
\newcommand{\diag}{\mathop{\rm diag}}

\newtheorem{theorem}{{\bf Theorem}}
\newtheorem*{theorem*}{{\bf Theorem}}
\newtheorem{lemma}{{\bf Lemma}}
\newtheorem{remark}{{\bf Remark}}

\newtheorem{assumption}{Assumption}
\usepackage{hyperref}

\begin{document}
%
\title{Risk Convergence of Centered Kernel Ridge Regression with Large Dimensional Data}
%
%
%

\author{Khalil Elkhalil,
Abla Kammoun, Xiangliang Zhang, Mohamed-Slim Alouini and Tareq Al-Naffouri
\thanks{The authors are with the Electrical Engineering Program, King Abdullah University of Science and Technology, Thuwal, Saudi Arabia; e-mails: \{khalil.elkhalil, abla.kammoun, xiangliang.zhang, tareq.alnaffouri, slim.alouini\}@kaust.edu.sa.}}

\maketitle

\begin{abstract}
\label{abstract}
This paper carries out a large dimensional analysis of a variation of kernel ridge regression that we call \emph{centered kernel ridge regression} (CKRR), also known in the literature as kernel ridge regression with offset. This modified technique is obtained by accounting for the bias in the regression problem resulting in the old kernel ridge regression but with \emph{centered} kernels. The analysis is carried out under the assumption that the data is drawn from a Gaussian distribution and heavily relies on tools from random matrix theory (RMT). Under the regime in which the data dimension and the training size grow infinitely large with fixed ratio and under some mild assumptions controlling the data statistics, we show that both the empirical and the prediction risks converge to a deterministic quantities that describe in closed form fashion the performance of CKRR in terms of the data statistics and dimensions. Inspired by this theoretical result, we subsequently build a consistent estimator of the prediction risk based on the training data which allows to optimally tune the design parameters. A key insight of the proposed analysis is the fact that asymptotically a large class of kernels achieve the same minimum prediction risk. This insight is validated with both synthetic and real data.
\end{abstract}


%
\IEEEpeerreviewmaketitle

\section{Introduction}
%
%
%
%
\label{introduction}
\IEEEPARstart{K}{ernel} ridge regression (KRR) is part of kernel-based machine learning methods that deploy a set of nonlinear functions to describe the real output of interest \cite{bishop, kernel_methods_Smola}. More precisely, the idea is to map the data into a high-dimensional space $\mathcal{H}$, a.k.a.  \emph{feature space}, which can even be of infinite dimension  resulting in a linear representation of the data with respect to the output. Then,  a linear regression problem is solved in $\mathcal{H}$ by controlling over-fitting with a regularization term. 
In fact, the most important advantage of kernel methods is the utilized \emph{kernel trick} or   \emph{kernel substitution} \cite{bishop}, which allows to directly work with kernels and avoid explicit use of feature vectors in $\mathcal{H}$. 
\par
Due to its popularity, a rich body of research has been conducted to analyze the performance of KRR. In \cite{alaoui_kernel}, a randomized version of KRR is studied with performance guarantees in terms of concentration bounds. The work in \cite{Rudi_NIPS2017} analyzes the random  features approximation in least squares kernel regression. More relevant results can be found in \cite{Giguere:2013} where upper bounds of the prediction risk have been derived in terms of the empirical quadratic risk
for general regression models. Similarly for KRR models,   an upper and lower bound on the expected risk have been provided in  \cite{Caponnetto2007} before being generalized to general regularization operators in \cite{Dicker2017KernelRV}.  Therefore, most of the results related to the performance analysis of KRR and related regression techniques are in the form of upper or lower bounds of the prediction risk. In this work, we study the problem from an asymptotic analysis perspective.
As we will demonstrate in the course of the paper, such an analysis brought about novel results that predict in an accurate fashion prediction risks metrics. 
Our focus is on a variation  of KRR called centered kernel ridge regression (CKRR) that is built upon the same principles of KRR with the additional requirement to minimize the \emph{bias} in the learning problem. This variation has been motivated by  Cortes et al. in \cite{cortes_centered_kernel} and \cite{centering1, centering2} where the benefits of centering kernels have been highlighted. 
The obtained regression technique can be seen as KRR with centered kernels. Moreover, in the high dimensional setting with certain normalizations, we show that kernel matrices behave as a rank one matrix, thus centering allows to neutralize this non-informative component and highlight higher order components that retain useful information of the data. 
\par 
To understand the behavior of CKRR, we conduct theoretical analysis in the large dimensional regime where both the data dimension $p$ and the training size $n$ tend to infinity with fixed ratio ($p/n \to \text{constant}$). As far as \emph{inner-product kernels} are concerned, with mild assumptions on the data statistics, we show using fundamental results from random matrix theory elaborated in \cite{karoui_kernel} and \cite{romain_clustering} that both the empirical and prediction risks approach a
deterministic quantity that relates in closed form fashion these performance measures to the data statistics and dimensions. This important finding allows to see how the model performance behaves as a function of the problem's parameters and as such tune the design parameters to minimize the prediction risk. Moreover, as an outcome of this result,  we show that it is possible to jointly optimize the regularization parameter along with the kernel function so that to achieve the possible minimum
prediction risk. In other words, the minimum prediction risk is always attainable for all kernels with a proper choice of the regualrization parameter. This implies that all kernels behave similarly to the linear kernel. We regard such a fact  as a consequence of the curse of dimensionality phenomenon which causes the CKRR to be asymptotically equivalent to \emph{centered linear ridge regression}. As an additional contribution of the present work, we  build a consistent estimator of the prediction risk based on the training samples, thereby paving the way towards  optimal setting of the regularization parameter. 

The rest of the paper is structured as follows. In section \ref{background KRR}, we give a brief background on kernel ridge regression and introduce its centered variation. In section \ref{main results}, we provide the main results of the paper related to the asymptotic analysis of CKRR as well as the construction of a consistent estimator of the prediction risk. Then, we provide some numerical examples in section \ref{experiments}. We finally make some concluding remarks in section \ref{conclusion}.
\subsection*{Notations: $\E \left[.\right]$ and $\text{var}\left[.\right]$ stand for the expectation and the variance of a random variable while $\to_{a.s.}$ and $\to_{prob.}$ respectively stand for the almost sure convergence and the convergence in probability. $\left \| .\right \|$ denotes the operator norm of a matrix and the $L_2$ norm for vectors, $\tr \left[.\right]$ stands for the trace operator. The notation $f = O\left(g\right)$ means that $\exists M$ bounded such that $f \leq M g$. We say that $f$ is $\mathcal{C}^p$ if the $p$th derivative of $f$ exists and is continous.}
\section{Background on kernel ridge regression} \label{background KRR}
Let $\{\left(\bm{x}_i, y_i\right)\}_{i=1}^n$ be a set of $n$ observations in $\mathcal{X} \times \mathcal{Y}$, where $\mathcal{X}$ denotes the input space and $\mathcal{Y}$ the output space. Our aim is to predict the output of new input points $\bm{x} \in \mathcal{X}$ with a reasonable accuracy. Assume that the output is generated using a function $f: \mathcal{X} \to \mathcal{Y}$, then the problem can be cast as a function approximation problem where the goal is to find an estimate of $f$ denoted by $\widehat{f}$ such that $\widehat{f}\left(\bm{x}\right)$ is close to the real output $f\left(\bm{x}\right)$. In this context, the kernel learning problem is formulated as follows
\begin{equation}
\label{kernel_optim_pb}
\min_{f \in \mathcal{H}} \frac{1}{2} \sum_{i=1}^{n} l\left(y_i, f\left(\bm{x}_i\right)\right) + \frac{\lambda}{2} \left \| f \right \|_{\mathcal{H}}^2,
\end{equation}
where $\mathcal{H}$ is a reproducing kernel Hilbert space (RKHS), $l: \mathcal{Y} \times \mathcal{Y} \to \mathbb{R}$ is a loss function and $\lambda > 0$ is a regularization parameter that permits to control overfitting. Denoting by $\phi: \mathcal{X} \to \mathcal{H}$ a feature map that maps the data points to the feature space $\mathcal{H}$, then we define $k: \mathcal{X} \times \mathcal{X} \to \mathbb{R}$ such that ${k}\left(\bm{x}, \bm{x}'\right) = \left \langle
\phi\left(\bm{x}\right), \phi\left(\bm{x}'\right) \right \rangle_{\mathcal{H}}$ for all $\bm{x}, \bm{x}' \in \mathcal{X}$ where ${k}$ is known as the positive definite kernel corresponding to the feature map $\phi$. With these definitions, the representer theorem \cite{representer1, representer2} shows that the minimizer of the problem in \eqref{kernel_optim_pb} writes as $f^{*} \left(\bm{x}\right) = \bm{\alpha}^T \phi\left(\bm{x}\right)$. 
Thus, we can reformulate \eqref{kernel_optim_pb} as follows
\begin{equation}
\label{kernel_optim_K}
\min_{\bm{\alpha} \in \mathcal{H}} \frac{1}{2} \sum_{i=1}^{n} l\left(y_i, \bm{\alpha}^T \phi\left(\bm{x}_i\right)\right) + \frac{\lambda}{2} \left \| \bm{\alpha} \right \|^2.
\end{equation}
When $l$ is the squared loss, the optimization problem in \eqref{kernel_optim_K} can be reformulated as 
\begin{align}
\min_{\bm{\alpha} \in \mathcal{H}} \frac{1}{2} \left \| \bm{y} - \mathbf{\Phi} \bm{\alpha} \right \|^2 + \frac{\lambda}{2} \left \| \bm{\alpha} \right \|^2,
\end{align}
where $\mathbf{\Phi} = \left[\phi\left(\bm{x}_1\right), \cdots, \phi\left(\bm{x}_n\right)\right]^T$. This yields the following solution
$
\bm{\alpha}^* = \left(\mathbf{\Phi}^T \mathbf{\Phi} + \lambda \mathbf{I} \right)^{-1} \mathbf{\Phi}^T \bm{y}
$, 
where $\bm{y} = \{y_i\}_{i=1}^n$. Then, the output estimate of any data point $\bm{s}$ is given by \cite{bishop}
\begin{equation}
\label{KRR_estimate}
\begin{split}
\widehat{f}\left(\bm{s}\right)	= \bm{\kappa}\left(\bm{s}\right)^T \left(\mathbf{K} + \lambda \mathbf{I}\right)^{-1} \bm{y},
\end{split}
\end{equation}
where $\bm{\kappa}\left(\bm{s}\right) = \{{k}\left(\bm{s}, \bm{x}_i\right)\}_{i=1}^n$ is the information vector and $\mathbf{K} = \mathbf{\Phi} \mathbf{\Phi}^T$ with entries $\mathbf{K}_{i,j} = {\bm \kappa}\left(\bm{x}_i, \bm{x}_j\right)$, $1 \leq i,j \leq n$. This is commonly known as the kernel trick which allows to highly simplify the problem which boils down to solving a $n$-dimensional problem.
Throughout this paper, we consider the following data model
\begin{align}
\label{data_model}
y_i = f\left(\bm{x}_i\right) + \sigma \epsilon_i, \quad i =1, \cdots, n.
\end{align}
where $f$ generates the actual output of the data and $\epsilon_i$ are i.i.d. standard normal random variables with $\sigma^2$ assumed to be known. 
We consider both the empirical (training) and the prediction (testing) risks respectively defined as \cite{vapnik_regression_vc}
\begin{align}
\label{risk_train}
& \mathcal{R}_{\rm train} \left(\widehat{f}\right)
=\frac{1}{n} \mathbb{E}_{\bm{\epsilon}} \|	\widehat{f}\left(\mathbf{X}\right) - 	f\left(\mathbf{X}\right)\|_2^2, \\  
\label{risk_test}
& \mathcal{R}_{\rm test}\left(\widehat{f}\right) =  \E_{\bm{s} \sim \mathcal{D},\bm{ \epsilon}}\left[\left(\widehat{f}\left(\bm{s}\right) - f\left(\bm{s}\right) \right)^2\right], 
\end{align}
where $\mathcal{D}$ is the data input distribution, $\bm{s}$ is taken independent of the training data $\mathbf{X}$ and $\bm{\epsilon} =\{\epsilon_i\}_{i=1}^n$. The above two equations respectively measure the goodness of fit relative to the training data and to new unseen data all in terms of the mean squared error (MSE). 
\subsection{Centered kernel ridge regression}
The concept of centered kernels  dates back to the work of Cortes \cite{cortes_centered_kernel} on learning kernels based on the notion of centered alignment. As we will show later, this notion of centering comes naturally to the picture when we account for the bias in the learning problem (also see the lecture notes by Jakkola \cite{Jakkola_CKRR}). More specifically, we modify the optimization problem in \eqref{kernel_optim_K} to account for the bias as follows
\begin{equation}
\label{kernel_optim_bias}
\min_{\alpha_0, \bm{\alpha} \in \mathcal{H}} \frac{1}{2} \sum_{i=1}^{n} l\left(y_i, \alpha_0 + \bm{\alpha}^T \phi\left(\bm{x}_i\right)\right) + \frac{\lambda}{2} \left \| \bm{\alpha} \right \|^2,	
\end{equation}
where clearly we do not penalize the \emph{offset} (or the \emph{bias}) $\alpha_0$ in the regularization term
. With $l$ being the squared loss, we immediately get 
$	\alpha_0^* = \bar{y} - \frac{1}{n}\bm{\alpha}^T\bm{\Phi}^T \mathbf{1} , \quad \overline{y} = \frac{1}{n} \bm{1}^T\bm{y}.$
Substituting $\alpha_0^*$ in \eqref{kernel_optim_bias}, we solve the centered optimization problem given by
\begin{equation}
\label{centered_problem}
\min_{\bm{\alpha} \in \mathcal{H}} \frac{1}{2}\left \| \mathbf{P} \left(\bm{y} - \mathbf{\Phi}\bm{\alpha}\right) \right \|_2^2 + \frac{\lambda}{2} \bm{\alpha}^T \bm{\alpha},		
\end{equation}
where $\mathbf{P} = \mathbf{I}_n - \frac{1}{n}\bm{1}_n\bm{1}_n^T$ is referred as a \emph{projection} matrix or a \emph{centering} matrix \cite{cortes_centered_kernel, Jakkola_CKRR}. Finally, we get 
\begin{align*}
\bm{\alpha}^* &  = \left(\mathbf{\Phi}^T \mathbf{P}\mathbf{\Phi} + \lambda \mathbf{I}\right)^{-1} \mathbf{\Phi}^T \mathbf{P} \left(\bm{y}-\bar{y}\mathbf{1}\right)  \\  & \overset{(b)}{=}\mathbf{\Phi}^T \mathbf{P} \left(\mathbf{P}\mathbf{\Phi} \mathbf{\Phi}^T \mathbf{P} + \lambda \mathbf{I}_n\right)^{-1}\left(\bm{y}-\bar{y}\mathbf{1}_n\right) \\ 
& = \mathbf{\Phi}^T \mathbf{P} \left(\mathbf{P}\mathbf{K} \mathbf{P} + \lambda \mathbf{I}_n\right)^{-1}\left(\bm{y}-\bar{y}\mathbf{1}_n\right) \\ 
& = \mathbf{\Phi}^T \mathbf{P} \left(\mathbf{K}_c + \lambda \mathbf{I}_n\right)^{-1}\left(\bm{y}-\bar{y}\mathbf{1}_n\right), 
\end{align*}
where $\mathbf{K}_c=\mathbf{P}\mathbf{K} \mathbf{P}$ is the \emph{centered} kernel matrix as defined in \cite[Lemma1]{cortes_centered_kernel} and $(b)$ is obtained using the Woodbury identity. With some basic manipulations, the centered kernel ridge regression estimate of the output of data point $\bm{s}$ is given by
\begin{equation}
\label{CKRR_estimate}
\widehat{f}_c\left(\bm{s}\right)	= \bm{\kappa}_c\left(\bm{s}\right)^T \left(\mathbf{K}_c + \lambda \mathbf{I}\right)^{-1} \mathbf{P}\bm{y} + \bar{y}.
\end{equation}
Therefore, the feature map corresponding to $\mathbf{K}_c$ as well as the information vector can be respectively obtained as follows
\begin{align}
\phi_c\left(\bm{s}\right)  = \phi\left(\bm{s}\right) - \frac{1}{n} \sum_{i=1}^{n} \phi\left(\bm{x}_i\right), \:
\bm{\kappa}_c\left(\bm{s}\right)  = \mathbf{P} \bm{\kappa}\left(\bm{s}\right) - \frac{1}{n} \mathbf{P} \mathbf{K} \mathbf{1}_n .
\end{align}
Throughout this paper, we consider inner-product kernels \cite{bishop, karoui_kernel} defined as follows
\begin{align}
\label{inner_kernels}
{k}\left(\bm{x}, \bm{x}'\right) = g\left(\bm{x}^T\bm{x}' / p\right), \quad \forall \bm{x}, \bm{x}' \in \mathbb{R}^p,
\end{align}
and subsequently, 
$
\mathbf{K} = \left \{ g\left(\bm{x}_i^T\bm{x}_j / p\right) \right \}_{i, j=1}^n,
$
where the normalization\footnote{This is equivalent to normalize all data points by $\sqrt{p}$. This type of normalization has also been conisdered in \cite{Valentini_bias_variance_SVM} following the heuristic of Jakkola.} by $p$ in \eqref{inner_kernels} is convenient in the large $n, p$ regime as we will show later (also see \cite{romain_svm} for similar normalization in the analysis of LS-SVMs).
In the following, we conduct a large dimensional analysis of the performance of CKRR with the aim to get useful insights on the design of CKRR. Particularly, we will focus on studying the empirical and the prediction risks of  CKRR which we define as
\begin{align*}
& \mathcal{R}_{\rm train}=\frac{1}{n}\mathbb{E}_{\bm \epsilon}\sum_{i=1}^n\left|\widehat{f}_c({\bm x}_i)-f({\bm x}_i)\right|^2 \hspace{0.5cm},  \\ & \mathcal{R}_{\rm test}=\mathbb{E}_{\bm s\sim \mathcal{D},{\bm \epsilon}} \left[\left|\widehat{f}_c({\bm s})-f({\bm s})\right|^2\right]. 
\end{align*}

The novelty of our analysis with respect to previous studies lies in that 
\begin{enumerate}
\item It  provides a mathematical connection between the performance and the problem's dimensions and statistics resulting in a deeper understanding of centered kernel ridge regression in the large $n, p$ regime. 
\item It brings insights on how to choose the kernel function $g$ and the regularization parameter $\lambda$ in order to guarantee a good generalization performance for unknown data. 
\end{enumerate}
As far as the second point is considered, we show later that both the kernel function and the regularization parameter can be optimized jointly as a consequence of the mathematical result connecting the prediction risk with these design parameters. Our analysis does not assume a specific choice of the inner-product kernels, and is valid for the following popular ones.
\begin{itemize}
\item Linear kernels: $	k\left(\bm{x}, \bm{x}'\right) = \alpha \bm{x}^T \bm{x}'/p + \beta$. 
\item Polynomial kernels: $	k\left(\bm{x}, \bm{x}'\right) = \left(\alpha \bm{x}^T \bm{x}'/p + \beta \right)^d$. 
\item Sigmoid kernels: $	k\left(\bm{x}, \bm{x}'\right) = \tanh\left(\alpha \bm{x}^T \bm{x}'/p + \beta \right)$.
\item Exponential kernels: $	k\left(\bm{x}, \bm{x}'\right) = \exp\left(\alpha \bm{x}^T \bm{x}'/p + \beta \right)$.
\end{itemize}
\section{Main results} \label{main results}
\subsection{Technical assumptions}
In this section, we will present our theoretical results on the prediction risk of CKRR by first introducing the assumptions of data growth rate, kernel function $g$ and true function $f$.
Without loss of generality, we  assume that the data samples $\bm{x}_1, \cdots, \bm{x}_n \in \mathbb{R}^p$ are independent such that $\bm{x}_i \sim \mathcal{N}\left(\bm{0}_p, \mathbf{\Sigma}\right), i=1, \cdots, n$, with positive definite covariance matrix $\mathbf{\Sigma} \in \mathbb{R}^{p \times p}$. Throughout the analysis, we consider the large dimensional regime in which both $p$ and $n$ grow simultaneously large with the following growth rate assumptions.
\begin{assumption}[Growth rate] \label{assumption1} 
As $p, n \to \infty$ we assume the following 
\begin{itemize}
\item \textbf{Data scaling}: $p/n \to c \in \left(0, \infty\right)$.
\item \textbf{Covariance scaling}: $\lim \sup_p \left \| \mathbf{\Sigma} \right \| < \infty$.
\end{itemize}
\end{assumption}
The above assumptions are standard to consider and allow to exploit the large heritage of random matrix theory. Moreover, allowing $p$ and $n$ to grow large at the same rate is of practical interest when dealing with modern large and numerous data. The  assumption treating the covariance scaling is technically convenient since it allows to use important theoretical results on the behavior of large kernel matrices \cite{karoui_kernel, romain_clustering}. Under Assumption \ref{assumption1}, we have the
following implications.
\begin{align}
\label{diag_conv}
\bm{x}_i^T\bm{x}_i/p & \to_{a.s.} \frac{1}{p} \tr \mathbf{\Sigma} \triangleq \tau, \: i=1, \cdots, n. \\
\label{off_diag_conv}
\bm{x}_i^T\bm{x}_j/p & \to_{a.s.} 0, \: i \neq j,
\end{align}
where $0 < \tau < \infty$ due to the covariance scaling in Assumption \ref{assumption1}. This means that in the limit when $p \to \infty$, the kernel matrix $\mathbf{K}$ as defined earlier has all its entries converging to a deterministic limit. Applying a Taylor expansion on the entries of ${\bf K}$, and under some assumption on the kernel function $g$, it has been shown in \cite[Theorem 2.1]{karoui_kernel} that 
\begin{align} \label{K_limit}
\left \| \mathbf{K} - \mathbf{K}^{\infty}\right \| \to_{prob.} 0,
\end{align}
where the convergence is in operator norm and $\mathbf{K}^{\infty}$ exhibits nice properties and can be expressed using standard random matrix models. The explicit expression of $\mathbf{K}^{\infty}$ as well as its properties will be thoroughly investigated in Appendix A. 
We subsequently make additional assumptions to control the kernel function $g$ and the data generating function $f$. 
\begin{assumption}[Kernel function]\label{assumption2}
As in \cite[Theorem 2.1]{karoui_kernel}, we shall assume that $g$ is $\mathcal{C}^1$ in a neighborhood of $\tau$ and $\mathcal{C}^3$ in a neighborhood of 0. Moreover, we assume that for any independent observations $\bm{x}_i$ and $\bm{x}_j$  drawn from $\mathcal{N}({\bf 0}_p,{\bm \Sigma})$ and $k\in \mathbb{N}^{*}$,
$$
\mathbb{E}\left|g^{(3)}\left(\frac{1}{p}\bm{x}_i^{T}\bm{ x}_j\right)\right|^k <\infty
$$
where $g^{(3)}$ is the third derivative of $g$. 
\end{assumption}
\begin{assumption}[Data generating function]\label{assumption3} We assume that $f$ is $\mathcal{C}^1$ and polynomially bounded together with its derivatives.
We shall further assume that the moments of  $f(\bm{x})$ and its gradient are finite. More explicitly we need to have:
$$
\mathbb{E}_{\bm x\sim \mathcal{N}({\bf 0},{\bm \Sigma})}\left|f(\bm{x})\right|^k <\infty,
$$
and
\begin{align}
\label{bounded_grad}
\E_{\bm{x} \sim\mathcal{N}({\bf 0},{\bm \Sigma}) } \left \|  \bm{\nabla}_f\left(\bm{x}\right)\right \|_2^k < \infty , \text{where} \: \bm{\nabla}_f\left(\bm{x}\right) = \left \{ \frac{\partial f\left(\bm{x}\right)}{\partial x_l} \right \}_{l=1}^p.
\end{align}
\end{assumption}

As we will show later, the above assumptions are needed to guarantee a bounded asymptotic risk and to carry out the analysis.  Under the setting of Assumptions \ref{assumption1}, \ref{assumption2} and \ref{assumption3}, we aim to study the performance of CKRR by asymptotically evaluating the performance metrics defined in \eqref{risk_train}. Inspired by the fundamental results from \cite{karoui_kernel} and \cite{romain_clustering} in the context of spectral clustering, then following the observations made in
\eqref{diag_conv} and\eqref{off_diag_conv}, it is always possible to linearize the kernel matrix $\mathbf{K}$ around the matrix $g\left(0\right) \bm{11}^T$ which avoids dealing with the original intractable expression of $\mathbf{K}$. Note that the first component of the approximation given by $g\left(0\right) \bm{11}^T$ will be neutralized by the projection matrix $\mathbf{P}$ in the context of CKRR, which means that the behavior of CKRR will be essentially governed by the higher order approximations of $\mathbf{K}$. Consequently, one can resort to those approximations to have an explicit expression of the asymptotic risk in the large $p,n$ regime. This expression would hopefully reveal the mathematical connection between the regression risk and the data' statistics and dimensions as $p,n \to \infty$. 
\subsection{Limiting risk}
With the above assumptions at hand, we are now in a position to state the main results of the paper related to the derivation of the asymptotic risk of CKRR. Before doing so, we shall introduce some useful quantities. 
\begin{align*}
\nu & \triangleq g\left(\tau\right) - g\left(0\right) - \tau g'\left(0\right). \\  
\end{align*}
Also, for all $z \in \mathbb{C}$ at macroscopic distance from the eigenvalues $\lambda_1,\cdots,\lambda_p$ of $\frac{1}{p}\sum_{i=1}^n\bm{x}_i\bm{x}_i^{T}$, we define the \emph{Stieltjes transform} of  $\frac{1}{p}\sum_{i=1}^n\bm{x}_i\bm{ x}_i^{T}$  also known as the \emph{Stieltjes transform} of the Mar\u{c}enko-Pastur law as the unique solution to the following fixed-point equation \cite{couillet}
\begin{align}
\label{stieltjes}
{m}\left(z\right) & \triangleq - \left[cz - \frac{1}{n} \tr \mathbf{\Sigma} \left(\mathbf{I} + {m}\left(z\right) \mathbf{\Sigma}\right)^{-1}\right]^{-1},
\end{align}
where ${m}\left(z\right)$ in \eqref{stieltjes} is bounded as $p \to \infty $ provided that Assumption \ref{assumption1} is satisfied. 
For ease of notation, we shall use ${m}_z$ to denote ${m}\left(z\right)$ for all appropriate $z$.
The first main result of the paper is summarized in the following theorem, the proof of which is postponed to the Appendix A.
\begin{theorem}[Limiting risk]
\label{theorem1}	
Under Assumptions \ref{assumption1}, \ref{assumption2} and \ref{assumption3} and by taking $z=-\frac{\lambda + \nu}{g'\left(0\right)}$ for kernel functions satisfying $g'\left(0\right) \neq 0$ \footnote{The case of $g'\left(0\right)=0$ is asymptotically equivalent to take the sample mean as an estimate of $f$ which is neither of practical nor theoretical interest.} and $z$ at macroscopic distance from the eigenvalues of $\frac{1}{p}\sum_{i=1}^n\bm{x}_i\bm{x}_i^{T}$, both the empirical and the prediction risks converge in probability to a non trivial deterministic limits respectively given by 
\begin{align}
\label{risk_asy}
& \mathcal{R}_{train} - \mathcal{R}_{train}^{\infty} \to_{prob.}0, \\ 
& \mathcal{R}_{test} - \mathcal{R}_{test}^{\infty} \to_{prob.}0,
\end{align}
where the expressions of $\mathcal{R}_{train}^{\infty}$ and $\mathcal{R}_{test}^{\infty}$ are given in the top of the next page.
\begin{figure*}[t!]
\begin{equation}
\label{R_train_inf}
\begin{split}
\mathcal{R}_{train}^{\infty} & = \left(\frac{c \lambda {m}_z}{g'\left(0\right)}\right)^2 	\frac{n\sigma^2 + n\textbf{var}_f - n{m}_z  \E\left[\bm{\nabla}_f\right]^T \mathbf{\Sigma} \left[\left(\mathbf{I}+ {m}_z\mathbf{\Sigma}\right)^{-1} + \left(\mathbf{I}+ {m}_z\mathbf{\Sigma}\right)^{-2}\right] \mathbf{\Sigma} \E\left[\bm{\nabla}_f\right]}{n-{m}^2_z \tr \bm{\Sigma}^2 \left(\mathbf{I} + {m}_z\bm{\Sigma}\right)^{-2}}  
+ \sigma^2 - 2 \sigma^2 \frac{c\lambda {m}_z}{g'\left(0\right)}.
\end{split}
\end{equation}
\begin{align}
\label{R_test_inf}
\mathcal{R}^{\infty}_{test} & = \frac{n\sigma^2 + n\textbf{var}_f - n{m}_z  \E\left[\bm{\nabla}_f\right]^T \mathbf{\Sigma} \left[\left(\mathbf{I}+ {m}_z\mathbf{\Sigma}\right)^{-1} + \left(\mathbf{I}+ {m}_z\mathbf{\Sigma}\right)^{-2}\right] \mathbf{\Sigma} \E\left[\bm{\nabla}_f\right]}{n-{m}^2_z \tr \bm{\Sigma}^2 \left(\mathbf{I} + {m}_z\bm{\Sigma}\right)^{-2}}
-\sigma^2.
\end{align}
\end{figure*}
\end{theorem}
Note that in the case where $\mathbf{\Sigma} = \mathbf{I}_p$, the limiting risks in \eqref{R_train_inf} and \eqref{R_test_inf} can be further simplified as 
\begin{align*}
\mathcal{R}_{train}^{\infty}  &  = \left(\frac{c \lambda {m}_z}{g'\left(0\right)}\right)^2  \\ & \times \frac{n \left(1+{m}_z\right)^2 \left(\sigma^2 + \textbf{var}_f\right)   - n{m}_z\left(2+{m}_z\right) \left \| \E\left[\bm{\nabla}_f\right] \right \|^2}{n\left(1+{m}_z\right)^2 - p{m}^2_z}   \\ & + \sigma^2 - 2 \sigma^2 \frac{c\lambda {m}_z}{g'\left(0\right)}.
\end{align*}
\begin{align*}
\mathcal{R}^{\infty}_{test}  & =  \frac{n \left(1+{m}_z\right)^2 \left(\sigma^2 + \textbf{var}_f\right) - n{m}_z\left(2+{m}_z\right) \left \| \E\left[\bm{\nabla}_f\right] \right \|^2}{n\left(1+{m}_z\right)^2 - p{m}^2_z}  \\ & -\sigma^2,
\end{align*}
where ${m}_z$ can be explicitly derived as in \cite{silverstein_book} 
$$
{m}_z = \frac{- \left(cz-c+1\right) - \sqrt{\left(cz-c-1\right)^2-4c}}{2cz}.
$$ 

\begin{remark}
\label{remark1}
From Theorem \ref{theorem1} it entails that the limiting prediction risk can be expressed using the limiting empirical risk in the following fashion.
\begin{align}
\label{test_relation_training}
\mathcal{R}^{\infty}_{test} = \left(\frac{c \lambda {m}_z}{g'\left(0\right)}\right)^{-2} \mathcal{R}^{\infty}_{train} - \sigma^2 \left(\frac{g'\left(0\right)}{c \lambda {m}_z} - 1\right)^2.
\end{align}
\end{remark}
\begin{lemma}[A consistent estimator of the prediction risk]
\label{lemma2}
Inspired by the outcome of Theorem \ref{theorem1} summarized in Remark \ref{remark1}, we construct a consistent estimator of the prediction risk given by 
\begin{align}
\label{test_estimate}
\widehat{\mathcal{R}}_{test} =  \left(\frac{c \lambda \widehat{{m}}_z}{g'\left(0\right)}\right)^{-2} \mathcal{R}_{train} - \sigma^2 \left(\frac{g'\left(0\right)}{c \lambda \widehat{{m}}_z} - 1\right)^2, \: \text{with} \: \lambda > 0,
\end{align}
in the sense that 
$	\mathcal{R}_{test} - \widehat{\mathcal{R}}_{test} \to_{prob.} 0,$
where ${m}_z$ can be consistently estimated as 
$	\widehat{{m}}_z = \frac{1}{p} \tr \left( \mathbf{XX}^T/p - z\mathbf{I}_n\right)^{-1}, \: \mathbf{X} = \left[\bm{x}_1, \cdots, \bm{x}_n\right]^T.$
\end{lemma}
\begin{proof}
The proof is straightforward relying on the relation in \eqref{test_relation_training} and the fact that 
\begin{align*}
\widehat{{m}}_z - {m}_z \to_{a.s.}0,
\end{align*}
as shown in \cite[Lemma 1]{romain_clustering}.
\end{proof}
Since the aim of any learning system is to design a model that achieves minimal prediction risk \cite{vapnik_regression_vc}, the relation described in Lemma \ref{lemma2} by \eqref{test_estimate} has enormous advantages as it permits to estimate the prediction risk in terms of the empirical risk and hence optimize the prediction risk accordingly. 
\begin{remark} \label{remark2} One important observation from the expression of the limiting prediction risk in \eqref{R_test_inf} is that the information on the kernel (given by $g'\left(0\right)$ and $\nu$) as well as the information on $\lambda$ are both encapsulated in ${m}_z$ with $z=-\frac{\lambda + \nu}{g'\left(0\right)}$. This means that one should optimize $z$ to have minimal prediction risk and thus jointly choose the kernel $g$ and the regularization parameter $\lambda$. Moreover, it
entails that the choice of the kernel (as long as $g'\left(0\right) \neq 0$) is asymptotically irrelevant since a bad choice of the kernel can be compensated by a good choice of $\lambda$ and vice-versa. 
This essentially implies that a linear kernel asymptotically achieves the same optimal performance as any other type of kernels \footnote{This does not mean that
all kernels will have the same performance for a given regularization parameter but means that they will achieve the same minimum prediction risk.}. 
\end{remark}

\subsection{A consistent estimator of the prediction risk}
Although the estimator provided in Lemma \ref{lemma2} permits to estimate the prediction risk by virtue of the empirical risk, it presents the drawback of being sensitive to small values of $\lambda$. In the following theorem, we provide a consistent estimator of the prediction risk constructed from the training data $\left \{ \left(\bm{x}_i, y_i\right) \right \}_{i=1}^n$ and  is less sensitive to small values of $\lambda$. 
\begin{theorem}[A consistent estimator of the prediction risk]
\label{consistent_estim}
Under Assumptions \ref{assumption1}, \ref{assumption2} and \ref{assumption3} with $g'\left(0\right) \neq 0$ and $z = -\frac{\lambda+\nu}{g'\left(0\right)}$, we construct a consistent estimator of the prediction risk based on the training data such that 
\begin{equation*}
\mathcal{R}_{test}-\widehat{\mathcal{R}}_{test} \to_{prob.}0,
\end{equation*}
\begin{equation}
\begin{split}
\label{consist_general}
\widehat{\mathcal{R}}_{test}  & = \frac{1}{\left(cz\widehat{{m}}_z\right)^2} \left[\frac{1}{np} \bm{y}^T \mathbf{PX} \left(z \widetilde{\mathbf{Q}}_z^2 -\widetilde{\mathbf{Q}}_z \right) \mathbf{X}^T\mathbf{P}\bm{y} + \textbf{var}\left(\bm{y}\right)\right]  \\ & - \sigma^2, 
\end{split}
\end{equation}
with $\widetilde{\mathbf{Q}}_z$ is the resolvent matrix given by
$
\widetilde{\mathbf{Q}}_z = \left(\frac{\mathbf{X}^T\mathbf{PX}}{p} - z\mathbf{I}_p\right)^{-1}	
$, 
with ${\bf X}=\left[\bm{x}_1,\cdots,\bm{x}_n\right]^{T}$. 
Moreover, in the special case where $\mathbf{\Sigma} = \mathbf{I}_p$, the estimator reduces to
\begin{equation}
\label{consist_identity}
\begin{split}
\widehat{\mathcal{R}}_{test}   & =  \frac{n \left(1+\widehat{{m}}_z\right)^2  \textbf{var}\left(\bm{y}\right) }{{n\left(1+\widehat{{m}}_z\right)^2 - p\widehat{{m}}^2_z}} \\ & -  \frac{\widehat{{m}}_z\left(2+\widehat{{m}}_z\right) \left[\bm{y}^T \mathbf{P}\frac{\mathbf{X}\mathbf{X}^T}{n} \mathbf{P}\bm{y}  -p\textbf{var}\left(\bm{y}\right)\right]}{n\left(1+\widehat{{m}}_z\right)^2 - p\widehat{{m}}^2_z}   -\sigma^2.
\end{split}
\end{equation}
\end{theorem}
Theorem \ref{consistent_estim} provides a generic way to estimate the prediction risk from the pairs of training examples $\left \{ \left(\bm{x}_i, y_i\right) \right \}_{i=1}^n$. This allows using the closed form expressions in \eqref{consist_general} and \eqref{consist_identity}\footnote{The expression in \eqref{consist_identity} is useful because it does not involve any matrix inversion unlike the one in \eqref{consist_general}.} with the same set of arguments in Remark \ref{remark2} to jointly estimate the optimal kernel and the optimal regularization parameter $\lambda$. 
\subsection{Parameters optimization}
We briefly discuss how to jointly optimize the kernel function and the regularization parameter $\lambda$. As mentioned earlier, we exploit the special structure in the expression of the consistent estimate $\widehat{\mathcal{R}}_{test}$ where both parameters (the kernel function $g$ and $\lambda$) are summarized in $z$. We focus on the case where $\mathbf{\Sigma} = \mathbf{I}_p$ due to the tractability of the expression of $\widehat{\mathcal{R}}_{test}$ in \eqref{consist_identity}. By simple calculations, we can show that $\widehat{\mathcal{R}}_{test}$ is minimized when $\widehat{m}_z$ satisifies the equation
\begin{align*}
 	p\text{var}\left(\bm{y}\right) \left[p \widehat{m}_z^2 + n\left(1+\widehat{m}_z\right)^2\right]   & =   A  \left(n + n\widehat{m}_z + p\widehat{m}_z^2 \right),
\end{align*}
where $A = \bm{y}^T \mathbf{P}\frac{\mathbf{X}\mathbf{X}^T}{n} \mathbf{P}\bm{y}$, 
which admits the following closed-form solution
\begin{equation}
	\label{optimal_mz}
	\begin{split}
	m^{\star} & = \frac{\sqrt{nA^2 - 4npA^2 + 8np^2A \text{var}\left(\bm{y}\right)  -4np^3\text{var}^2\left(\bm{y}\right)}}{2\left(-pA + np \text{var}\left(\bm{y}\right)  + p^2 \text{var}\left(\bm{y}\right) \right)} \\ & 
	-  \frac{nA - 2pn\text{var}\left(\bm{y}\right)}{\left(-pA + np \text{var}\left(\bm{y}\right)  + p^2 \text{var}\left(\bm{y}\right) \right)}.
		\end{split}
\end{equation}
Then, look up $z^{\star}$ such that $\widehat{m}_{z^{\star}} = m^{\star}$. Finally, choose $\lambda$ and $g\left(.\right)$ such that $z^{\star} = -\frac{\lambda + \nu}{g'\left(0\right)}$. In the general case, it is difficult to get a closed from expression in terms of $z$ or $\widehat{m}_z$, however it is possible to numerically optimize the expression of $\widehat{\mathcal{R}}_{test}$ with respect to $z$. This can be done using simple one dimensional optimization techniques implemented in most softwares. 
\section{Experiments}
\label{experiments}
\subsection{Synthetic data}
To validate our theoretical findings, we consider both Gaussian and Bernoulli data. As shown in Figure \ref{fig:risk_big}, both data distributions exhibit the same behavior for all the settings with different kernel functions. More importantly, eventhough the derived formulas heavily rely on the Gaussian assumption, in the case where the data is Bernoulli distributed, we have a good agreement with the theoretical limits. This can be understood as part of the universality property often encountered in many high dimensional settings. Therefore, we conjecture that the obtained results are valid for any data distribution following the model 
	$ \bm{x} \sim \mathbf{\Sigma}^{\frac{1}{2}} \bm{z},$ where $\mathbf{\Sigma}$ satisfies Assumption \ref{assumption1} and $\{\bm{z}_i \}_{1 \leq i \leq p }$ the entries of $\bm{z}$ are i.i.d. with zero mean, unit variance and have bounded moments\footnote{We couldn't provide experiments for more data distributions due to space limitations.}. For more clarity, we refer the reader to Figure \ref{fig:risk_rmt_estim} as a representative of Figure \ref{fig:risk_big} when the data is Gaussian with $p=100$ and $n=200$. 
As shown in Figure \ref{fig:risk_rmt_estim}, the proposed consistent estimators are able to track the real behavior of the prediction risk for all types of kernels into consideration. It is worth mentioning however that the proposed estimator in Lemma \ref{lemma2} exhibits some instability for small values of $\lambda$ due to the inversion of $\lambda$ in \eqref{test_estimate}. Therefore, it is advised to use the estimator given by Theorem \ref{consistent_estim}. It is also clear from Figure \ref{fig:risk_rmt_estim} that all the considered kernels achieve the same minimum prediction risks but with different optimal regularizations $\lambda$. This is not the case for the empirical risk as shown in Figure \ref{fig:risk_rmt_estim}  and \eqref{R_train_inf} where the information on the kernel and the regularization parameter $\lambda$ are decoupled. Hence, in contrast to the prediction risk, the regularization parameter and the kernel can not be jointly optimized
to minimize the empirical risk.
\subsection{Real data}
As a further experiment, we validate the accuracy of our result over a real data set.  
To this end, we use the real \emph{Communities and Crime Data Set} for evaluation \cite{uci-dataset}, 
which has 123 samples and 122 features. 
For the experiment in Figure \ref{fig:riskt_estim_communities}, we divide the data set to have $60 \%$ training samples ($n=73$) and the remaining for testing ($n_{test} = 50$). 
The risks in Figure \ref{fig:riskt_estim_communities} are obtained by averaging the prediction risk (computed using $n_{test}$) over $500$ random permutation of the data. Although the data set is far from being Gaussian, we notice that the proposed prediction risk estimators are able to track the real behavior of the prediction risk for all types of considered kernels. We can also validate the previous insight from Theorem \ref{theorem1} where all kernels almost achieve the same minimum prediction risk. 
\begin{figure}[H]
\centering{	\includegraphics[scale=0.36]{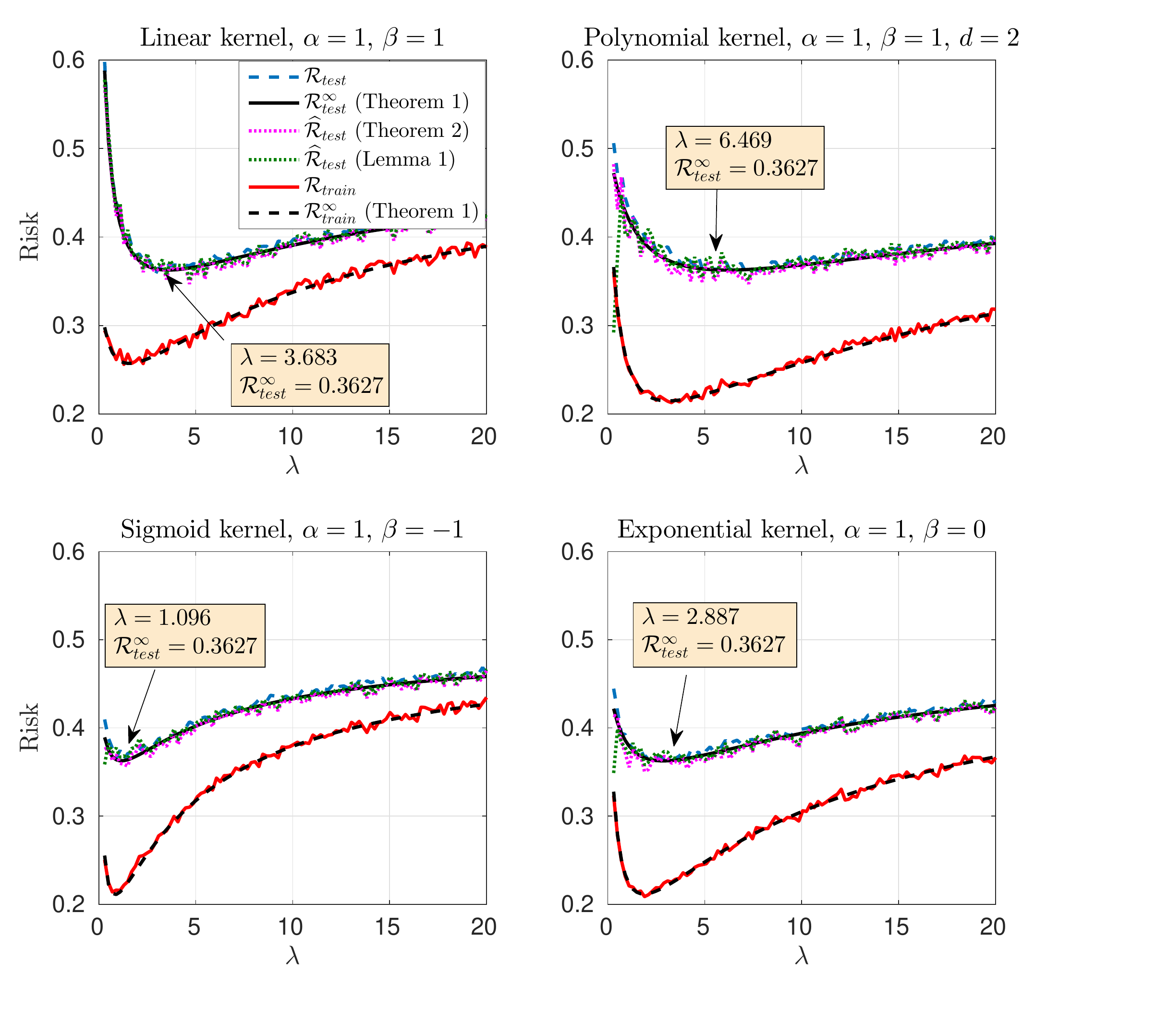}}
\caption{CKRR risk with respect to the regularization parameter $\lambda$ on Gaussian data ($\bm{x} \sim  \mathcal{N}\left(0_p, \{0.4^{|i-j|}\}_{i,j}\right)$, $n=200$ training samples with $p=100$ predictors, $\sigma^2 = 0.5$ for different types of kernels. The data generating function is taken to be $f\left(\bm{x}\right) = \sin\left(\bm{1}^T \bm{x} / \sqrt{p}\right)$.}
\label{fig:risk_rmt_estim}
\end{figure}
\begin{figure*}[t]
	\centering{	\includegraphics[scale=0.75]{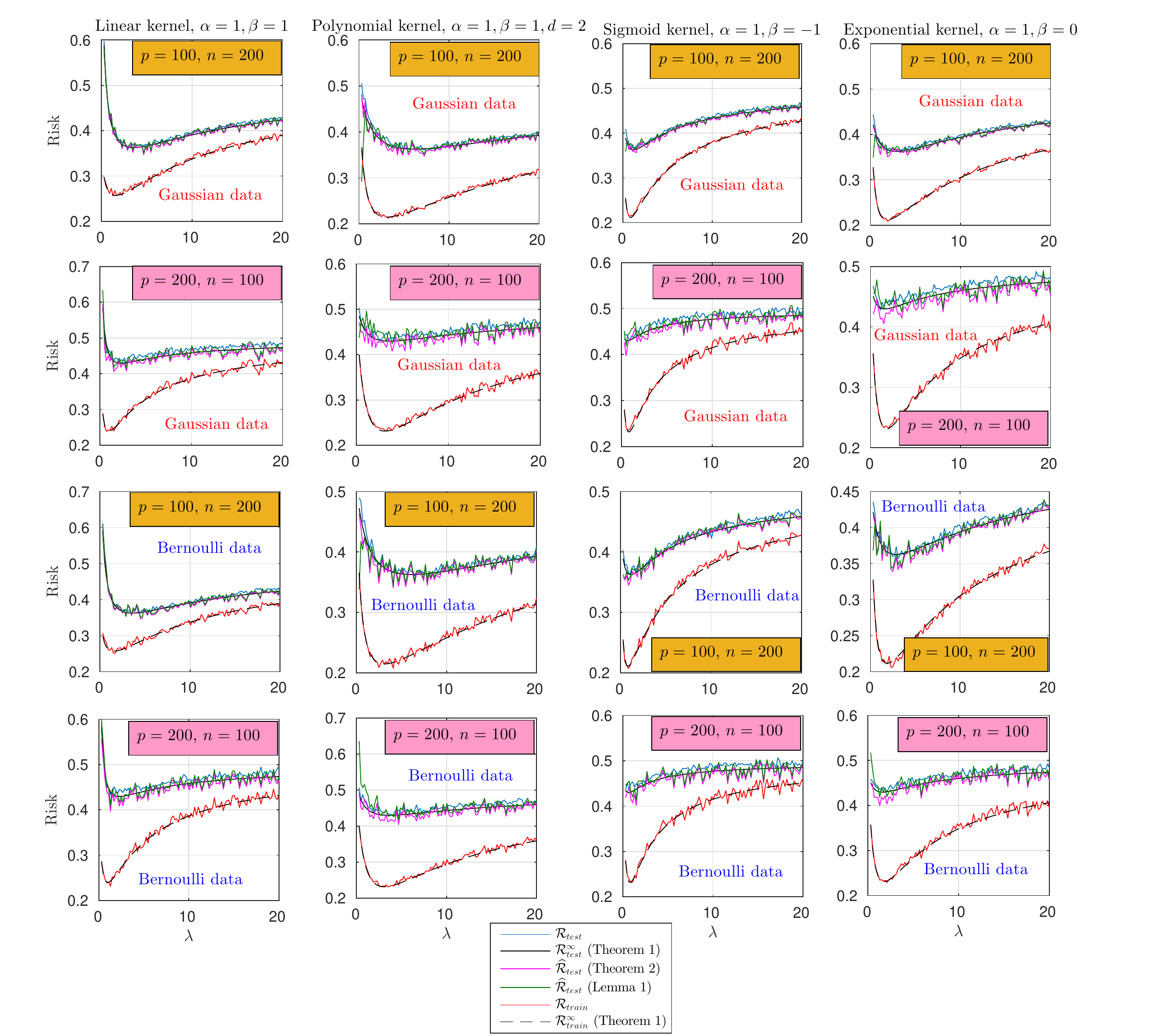}}
	\caption{CKRR risk with respect to the regularization parameter $\lambda$ on both Gaussian and Bernoulli data (i.e., $\bm{x} \sim  \mathcal{N}\left(0_p, \{0.4^{|i-j|}\}_{i,j}\right)$ and $\bm{x} = \{0.4^{|i-j|}\}_{i,j}^{\frac{1}{2}} \bm{z}$ with $\bm{z}_i \sim_{i.i.d.} 1-2 \times \text{Bernoulli} \left(1/2\right)$ respectively). The noise variance is taken to be $\sigma^2 = 0.5$ and the data generating function is $f\left(\bm{x}\right) = \sin\left(\bm{1}^T \bm{x} / \sqrt{p}\right)$.}
	\label{fig:risk_big}
\end{figure*}
\begin{figure}[h]
\centering
\includegraphics[scale=0.34]{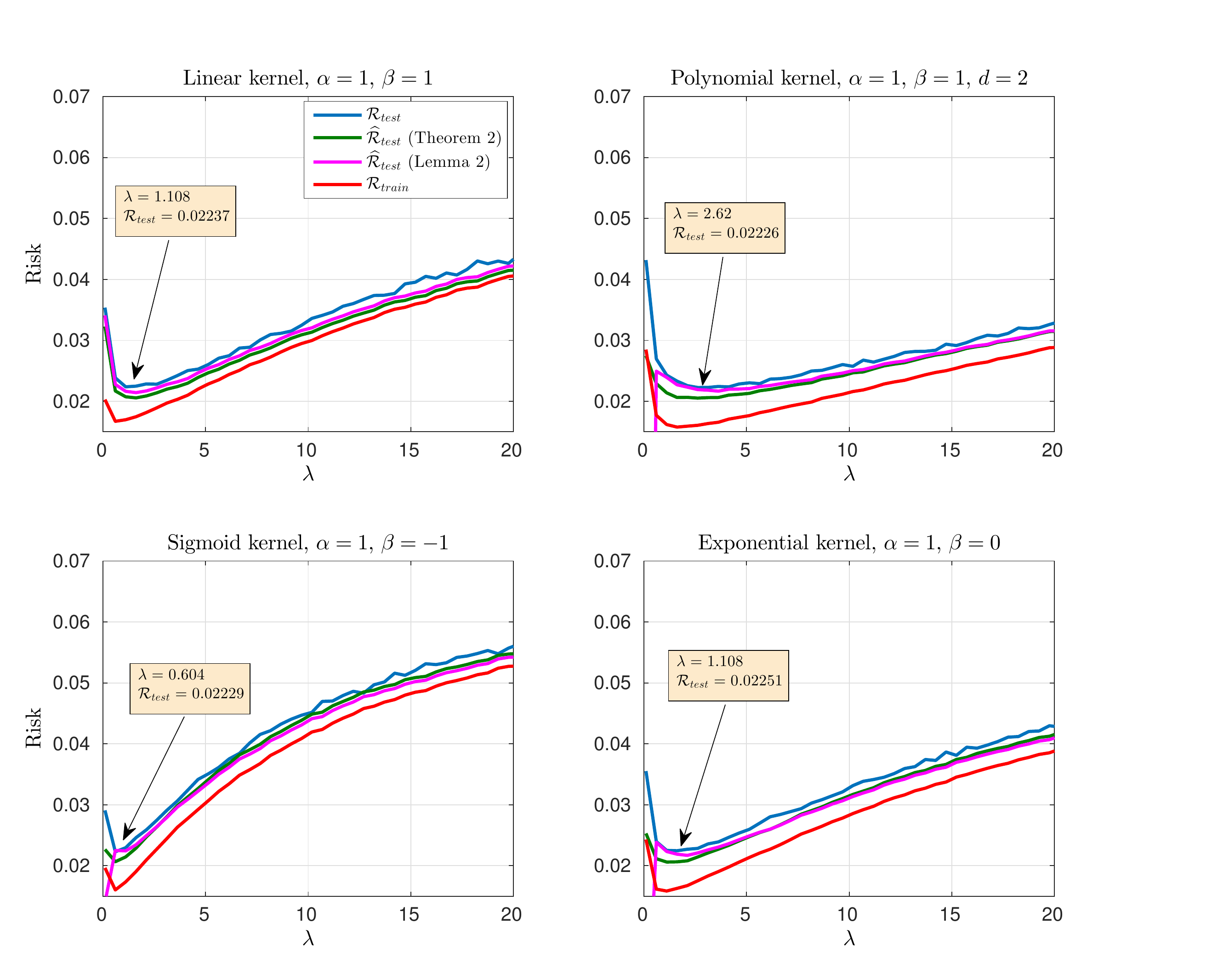}
\caption{CKRR risk with respect to the regularization parameter $\lambda$ on the \emph{Communities and Crime Data Set} where independent zero mean Gaussian noise samples with variance $\sigma^2=0.05$ are added to the true response.}
\label{fig:riskt_estim_communities}
\end{figure}
\section{Concluding Remarks}
\label{conclusion}
We conducted a large dimensional analysis of centered kernel ridge regression, which is a modified version of kernel ridge regression that accounts for the bias in the regression formulation. By allowing both the data dimension and the training size to grow infinitely large in a fixed proportion and by relying on fundamental tools from random matrix theory, we showed that both the empirical and the prediction risks converge in probability to a deterministic quantity that
mathematically connects these performance metrics to the data dimensions and statistics. A fundamental insight taken from the analysis is that asymptotically the choice of the kernel is irrelevant to the learning problem which asserts that a large class of kernels will achieve the same best performance in terms of prediction risk as a linear kernel. Finally, based
on the asymptotic analysis, we built a consistent estimator of the prediction risk making it possible to estimate the optimal regularization parameter that achieves  the minimum prediction risk. 
%
\section*{Appendix A}
\label{appendixA}
\section*{Proof of Theorem \ref{theorem1}}
Here, we provide the details of the derivation for the prediction risk. The analysis of the empirical risk follows in a very similar way and is thus omitted. 
Before delving into the proof of Theorem \ref{theorem1}, we shall introduce some fundamental results on the asymptotic behavior of inner-product kernel matrices established by El-Karoui \cite[Theorem 2.1]{karoui_kernel}. 
\begin{theorem*}[Asymptotic behavior of inner product kernel random matrices]
\label{karoui_inner} Under the assumptions of Theorem 2.1 \cite{karoui_kernel}, the kernel matrix $\mathbf{K}$ can be approximated by $\mathbf{K}^{\infty}$ in the sense that $\|\mathbf{K}-\mathbf{K}^{\infty}\| \to 0$ almost surely in operator norm, where
\begin{align*}
\mathbf{K}^{\infty} = \left[g\left(0\right) + g''\left(0\right)\frac{\tr \left(\mathbf{\Sigma}^2\right)}{2p^2} \right] \mathbf{1}\mathbf{1}^T + g'\left(0\right) \frac{\mathbf{XX}^T}{p} + \nu \mathbf{I}_n,
\end{align*}
where $\nu = g\left(\tau\right) - g\left(0\right)-\tau g'\left(0\right)$.
A similar result can be found in \cite{karoui_kernel} where the accuracy of ${\bf K}^{\infty}$ has been assessed as
$
{\bf K} ={\bf K}^{\infty} + O_{\|.\|}(\frac{1}{\sqrt{n}}),
\label{eq:K}
$
where $O_{\|.\|}(\frac{1}{\sqrt{n}})$ denotes a matrix with spectral norm converging in probability to zero with a rate $1/\sqrt{n}$. 
\end{theorem*}	
Define
\begin{align}
\mathbf{Q}_z = \left( \mathbf{P}\frac{\mathbf{XX}^T}{p} \mathbf{P} -z \mathbf{I}_n\right)^{-1}, \quad \widetilde{\mathbf{Q}}_z = \left(\frac{\mathbf{X}^T \mathbf{P}\mathbf{X}}{p} -z \mathbf{I}_p\right)^{-1}.
\end{align}
Note that using the Woodbury identity, it is easy to show the following useful relations  
\begin{align}
\label{Q_z_relation}
\mathbf{Q}_z & = -\frac{1}{z} \mathbf{I}_n + \frac{1}{zp}{\bf P} \mathbf{X} \widetilde{\mathbf{Q}}_z \mathbf{X}^T{\bf P},\\
\widetilde{\mathbf{Q}}_z &=-\frac{1}{z} \mathbf{I}_p +\frac{1}{zp} \mathbf{X}^{T}{\bf P}{\bf Q}_z{\bf PX}.
\end{align}
The above theorem has the following consequence
\begin{equation}
\begin{split}
\label{K_infty}
&   \left \| \left(\mathbf{K}_c + \lambda \mathbf{I} \right)^{-1}- \frac{1}{g'\left(0\right)} \left[ \mathbf{Q}_z + \frac{\frac{\nu}{g'\left(0\right)} \mathbf{Q}_z \frac{1}{n}\bm{11}^T \mathbf{Q}_z}{1 - \frac{\nu}{g'\left(0\right)} \frac{1}{n} \bm{1}^T \mathbf{Q}_z \bm{1}}\right] \right \|  \\ & =O_p(\frac{1}{\sqrt{n}}),
\end{split}
\end{equation}
where \eqref{K_infty} is obtained by a simple application of the Sherman-Morrison Lemma (inversion Lemma), along with the use of the resolvent identity ${\bf A}^{-1}-{\bf B}^{-1}= {\bf A}^{-1}({\bf B}-{\bf A}){\bf B}^{-1}$, which holds for any square invertible matrices ${\bf A}$ and ${\bf B}$. The proof of the above Theorem follows from the application of a Taylor expansion of the elements of $\frac{1}{p}{\bf XX}^{T}$ at the vicinity of their mean. Applying the same approach for vector 	${\bm \kappa}\left(\bm{s}\right)$, we get
$$
{\bm \kappa}({\bf s})= g(0){\bf 1}+ g'(0)\frac{1}{p}{\bf X}{\bf s} +\tilde{\bm \kappa}({\bf s}),
$$
where $\tilde{\bm \kappa}({\bf s})$ has elements \begin{equation}\left[\tilde{\bm \kappa}({\bf s})\right]_i=\frac{g''(0)}{2}\left(\frac{{\bf x}_i^{T}s}{p}\right)^2 +\frac{g^{3}(\xi_i) }{6}\left(\frac{{\bf x}_i^{T}{\bf s}}{p}\right)^3,
\label{eq:kappa_tilde}
\end{equation} with  $\xi_i\in \left[0,\frac{{\bf x}_i^{T}{\bf s}}{p}\right]$.  Then, since $\mathbb{E}\left|\frac{{\bf x}_i^{T}{\bf s}}{\sqrt{p}}\right|^{r}$ is uniformly bounded in $p$ for all $r\in\mathbb{N}$, we have for all $k\in \mathbb{N}$,  
\begin{equation}
\mathbb{E}\left\|\tilde{\bm \kappa}(\bf s)\right\|_2^{k}=O(\frac{1}{n^{\frac{k}{2}}}).
\label{eq:kappa_tilde}
\end{equation}
As shall be seen later, we need also to control $\mathbb{E}_{\bm s}\left[\tilde{\bm \kappa}(\bm s)\tilde{\bm \kappa}(\bm s)^{T}\right]$. 
This is performed in the following technical Lemma.
\begin{lemma}
\label{lemma:technical}
Let $\tilde{\bm \kappa}({\bm s})$ be as in \eqref{eq:kappa_tilde}. Then,
\begin{align*}
\mathbb{E}_{\bm s}\left[\tilde{\bm \kappa}(\bm s)\tilde{\bm \kappa}(\bm s)^{T}\right]&=\frac{1}{p^2}\left(\frac{g''(0)}{2}\right)^2\left(\frac{1}{p}\tr {\bm \Sigma}^2\right)^2{\bf 11}^{T}
+O_{\|.\|}(\frac{1}{n^{\frac{3}{2}}}).
\end{align*}
Similarly, the following approximations hold true
\begin{align*}
&\mathbb{E}_{\bm s}\left[{\bf X}{\bm s}\tilde{\bm \kappa}({\bm s})^{T}\right]=O_{\|.\|}(n^{-\frac{3}{2}}).\\
&\mathbb{E}_{\bm s}\left[\tilde{\bm \kappa}({\bm s })\frac{{\bf 1}^{T}}{n}{\bf KP}\right]=\frac{1}{p}\tr {\bm \Sigma} \frac{1}{n} \frac{{\bf 11}^{T}}{p}\frac{{\bf XX}^{T}}{p}{\bf P}+O_{\|.\|}(n^{-\frac{3}{2}}).
\end{align*}
\end{lemma}
\begin{proof}
To begin with, note that for ${\bf M}=\left\{m_{ij}\right\}_{i,j=1}^n$ a random matrix whose elements satisfies, $m_{i,j}=O_p(n^{-\alpha})$ for some $\alpha>0$, as $\|{\bf M}\|^2\leq \tr {\bf M}{\bf M}^{T}$, we have 
$
m_{i,j}=O_p(n^{-\alpha})\Rightarrow \|{\bf M}\|=O_p(n^{1-\alpha})\label{eq:res}
$. 
We first start by deriving $\mathbb{E}_{\bm s}\left[\left[\tilde{\bm \kappa}({\bf s})\right]_i\left[\tilde{\bm \kappa}({\bf s})\right]_j\right]$. We have
\begin{align*}
&\mathbb{E}_{\bm s}\left[\left[\tilde{\bm \kappa}({\bm s})\right]_i\left[\tilde{\bm \kappa}({\bm s})\right]_j\right]=\left(\frac{g''(0)}{2p^2}\right)^2\mathbb{E}_{\bm s}\left[{\bm s}^{T}{\bf x}_i{\bf x}_i^{T}{\bm  s}{\bm s}^{T}{\bf x}_j{\bf x}_j^{T}{\bm  s}\right]\\
&+\frac{g''(0)}{2}\mathbb{E}_{\bm s}\left[\left(\frac{{\bf x}_i^{T}{\bm s}}{p}\right)^2\left(\frac{{\bf x}_j^{T}{\bm s}}{p}\right)^3\frac{ g^3(\xi_i)}{6}\right] \\ &+\frac{g''(0)}{2}\mathbb{E}_{\bm s}\left[\left(\frac{{\bf x}_j^{T}{\bm s}}{p}\right)^2\left(\frac{{\bf x}_i^{T}{\bm s}}{p}\right)^3\frac{ g^3(\xi_j)}{6}\right]\\
&+\mathbb{E}_{\bm s}\left[\frac{g^{3}(\xi_i)}{6}\frac{g^{3}(\xi_j)}{6}\left(\frac{{\bf x}_i^{T}{\bm s}}{p}\right)^3\left(\frac{{\bf x}_j^{T}{\bm s}}{p}\right)^3\right]. 
\end{align*}
Using Assumption \ref{assumption2}, we can prove that $\mathbb{E}|g^{3}(\xi_i)|^{r}$ is bounded for all $r\in\mathbb{N}$. 
$
\mathbb{E}\left[\left(\frac{{\bf x}_i^{T}{\bm s}}{p}\right)^2\left|\frac{{\bf x}_j^{T}{\bm s}}{p}\right|^3| \frac{g^3(\xi_i)}{6}|\right]=O(n^{-\frac{5}{2}}). 
$
and
$
\mathbb{E}\left[\left|\frac{g^{3}(\xi_i)}{6}\right|\left|\frac{g^{3}(\xi_j)}{6}\right|\left|\frac{{\bf x}_i^{T}{\bm s}}{p}\right|^3\left|\frac{{\bf x}_j^{T}{\bm s}}{p}\right|^3\right]=O(n^{-3}).
$
Hence, by computing the expectation over ${\bm s} $ of the first term, we obtain
\begin{align*}
& \mathbb{E}_{\bm s}\left[\left[\tilde{\bm \kappa}({\bm s})\right]_i\left[\tilde{\bm \kappa}({\bm s})\right]_j\right]  =\frac{2}{p^2}\left(\frac{g''(0)}{2}\right)^2\left(\frac{{\bf x}_i^{T}{\bm \Sigma}{\bf x}_j}{p}\right)^2 \\ & + \frac{1}{p^2}\left(\frac{g''(0)}{2}\right)^2\frac{1}{p^2}{\bf x}_i^{T}{\bm \Sigma}{\bf x}_i {\bf x}_j^{T}{\bm \Sigma}{\bf x}_j   +\textcolor{black}{O_p(n^{-\frac{5}{2}})}.
\end{align*}
When $i\neq j$, $\left(\frac{{\bf x}_i^{T}{\bm \Sigma}{\bf x}_j}{p}\right)^2=O_p(n^{-1})$. 
Hence,
\begin{align*}
\mathbb{E}_{\bm s}\left[\left[\tilde{\bm \kappa}({\bm s})\right]_i\left[\tilde{\bm \kappa}({\bm s})\right]_j\right] & =\frac{2}{p^2}\left(\frac{g''(0)}{2}\right)^2\boldsymbol{\delta}_{i=j}\left(\frac{1}{p}{\bf x}_i^{T}{\bm \Sigma}{\bm x}_i\right)^2 \\ & +\frac{1}{p^2}\left(\frac{g''(0)}{2}\right)^2\frac{1}{p^2}{\bf x}_i^{T}{\bm \Sigma}{\bf x}_i {\bf x}_j^{T}{\bm \Sigma}{\bf x}_j \\ & +\textcolor{black}{O_p(n^{-\frac{5}{2}})}.
\end{align*}


Using \eqref{eq:res}, we thus obtain
\begin{align*}
& \mathbb{E}_{\bm s}\left[\tilde{\bm \kappa}({\bm s})\tilde{\bm \kappa}({\bm s})^{T}\right]\\ 
& =\frac{2}{p^2}\left(\frac{g''(0)}{2}\right)^2 \diag \left\{\left(\frac{1}{p}{\bf x}_i^{T}{\bm \Sigma}{\bf x}_i\right)^2\right\}_{i=1}^n \\ & +\frac{1}{p^2}\left(\frac{g''(0)}{2}\right)^2\left\{\frac{1}{p}{\bf x}_i^{T}{\bm \Sigma}{\bf x}_i\right\}_{i=1}^n\left(\left\{\frac{1}{p}{\bf x}_j^{T}{\bm \Sigma}{\bf x}_j\right\}_{j=1}^n\right)^{T}\\
&+O_{\|.\|}(n^{-\frac{3}{2}}).
\end{align*}
It is easy to see that $\left\|\frac{2}{p^2}\left(\frac{g''(0)}{2}\right)^2 \diag \left\{\left(\frac{1}{p}{\bf x}_i^{T}{\bm \Sigma}{\bf x}_i\right)^2\right\}_{i=1}^n\right\|=\textcolor{black}{O_p(n^{-2})}$. On the other hand, one can show that we can replace in the second term $\frac{1}{p}{\bf x}_i^{T}{\bm \Sigma} {\bf x}_i$ by $\frac{1}{p}\tr {\bm \Sigma}^2$. This is because
\begin{align*}
& \left\|\frac{1}{p^2}\left\{\frac{1}{p}{\bf x}_i^{T}{\bm \Sigma}{\bf x}_i-\frac{1}{p}\tr {\bm \Sigma}^2\right\}_{i=1}^n\left(\left\{\frac{1}{p}{\bf x}_j^{T}{\bm \Sigma}{\bf x}_j\right\}_{j=1}^n\right)^{T}\right\|_2 \\ &\leq \frac{1}{p^{\frac{3}{2}}} \left\|\left\{\frac{1}{p}{\bf x}_i^{T}{\bm \Sigma}{\bf x}_i-\frac{1}{p}\tr {\bm \Sigma}^2\right\}_{i=1}^n\right\|_2\left\|\left\{\frac{1}{p^{\frac{3}{2}}}x_j^{T}{\bm \Sigma}{\bf x}_j\right\}_{j=1}^n\right\|_2\\
&=O_p(n^{-\frac{3}{2}}).
\end{align*}
Putting all the above results together, we obtain
\begin{align*}
\mathbb{E}_{\bm s}\left[\tilde{\bm \kappa}({\bm s})\tilde{\bm \kappa}({\bm s})^{T}\right] & =\frac{1}{p^2}\left(\frac{g''(0)}{2}\right)^2\left(\frac{1}{p}\tr {\bm \Sigma}^2\right)^2{\bf 11}^{T} \\ & +O_{\|.\|}(n^{-\frac{3}{2}})
.
			\end{align*}
\end{proof}
Now using the approximation in \eqref{eq:K}, we obtain
\begin{equation}
{\bm \kappa}_c({\bf s}) = g'(0)\frac{1}{p}{\bf PXs} +{\bf P} \tilde{\bm \kappa}({\bf s}) -\frac{1}{n}{\bf PK1}.
\label{eq:kappa_c}
\end{equation}
\begin{theorem*}[Asymptotic behavior of $\mathbf{Q}_z$ and $\widetilde{\mathbf{Q}}_z$] As in \cite[Lemma 1]{romain_clustering}, let Assumption \ref{assumption1} holds, then as $p \to \infty$ and all $z\in \mathbb{C}\backslash \mathbb{R}_{+}$, 
\begin{align}
\label{Q_tilde_limit}
\widetilde{\mathbf{Q}}_z & \leftrightarrow  -\frac{1}{z}   \left( \mathbf{I} + m_{z} \mathbf{\Sigma}\right)^{-1}, \forall z \in \mathbb{C} \backslash \textbf{supp}\left(\mathbf{\Sigma}\right),
\end{align}
where $m_z$ is the unique  stieltjes transform solution, for all such $z$, to the implicit equation
$$
m_z= -\left(cz-\frac{1}{n}\tr {\bm \Sigma} \left( \mathbf{I} + m_{z} \mathbf{\Sigma}\right)^{-1}\right)^{-1},
$$
and the notation $ \mathbf{A} \leftrightarrow \mathbf{B}$ means that as $p \to \infty$, $\frac{1}{p} \tr \mathbf{M\left( A-B \right)} \to_{a.s.}0$ and $\bm{u}^T \left(\mathbf{A}-\mathbf{B}\right) \bm{v} \to_{a.s.}0$, for all deterministic Hermitian matrices $\mathbf{M}$ and deterministic vectors $\bm{u}, \bm{v}$ of bounded norms. 
Moreover, from \cite{walid_new} and \cite{bilinear} and $z\in\mathbb{C}\backslash\mathbb{R}_{+}$,
\begin{align}
&\frac{1}{n}\tr {\bf MQ}_z + \frac{1}{z}  \tr {\bf M} \left( \mathbf{I} + m_{z} \mathbf{\Sigma}\right)^{-1}=\frac{1}{n}\psi_n(z), \label{eq:trace}\\
&{\bm u}^T{\bf Q}_z{\bm v} + \frac{1}{z}  {\bm u}^{T}\left( \mathbf{I} + m_{z} \mathbf{\Sigma}\right)^{-1}{\bm v}=\frac{1}{\sqrt{n}}h_n(z), \label{eq:quad}
\end{align}
where for all $k\in\mathbb{N}$, $\mathbb{E}\left|\psi_n(z)\right|^k$ and $\mathbb{E}\left|h_n(z)\right|^k$ can be bounded uniformly in $n$ over any compact at a macroscopic distance from the limiting support of $\frac{1}{p}{\bf PXX}^{T}{\bf P}$.    
\label{th:approx}
\end{theorem*}
\begin{theorem}[An Integration by parts formula for Gaussian functionals] \cite{walid_new}
	With $f$ satisfying Assumption \ref{assumption3} and for $\bm{x} = \left[x_1, \cdots, x_p\right]^T \sim \mathcal{N}\left(\mathbf{0}_p, \mathbf{\Sigma}\right)$, we have 
	\begin{align}
	\mathbb{E} \left[ x_i f\left(\bm{x}\right)\right] = \sum_{j=1}^{p} \left[\mathbf{\Sigma} \right]_{i,j} \E \left[\frac{\partial f\left(\bm{x}\right)}{\partial x_j}\right],
	\end{align}
	or equivalently, 
$
	\mathbb{E} \left[\bm{x} f\left(\bm{x}\right)\right] = \mathbf{\Sigma} \E \left[\nabla f\left(\bm{x}\right)\right]. 
$
	\label{th:integration_part}
	\end{theorem}
	\begin{theorem*}[Nash-Poincar\'e inequality]  \cite{walid_new}
		With $f$ satisfying Assumption \ref{assumption3} and for $\bm{x} = \left[x_1, \cdots, x_p\right]^T \sim \mathcal{N}\left(\mathbf{0}_p, \mathbf{\Sigma}\right)$, we have under the setting of the previous theorem,
		$
		{\rm var}({ f}({\bf x})) \leq \mathbb{E}\left[\nabla f({\bf x})^{T}{\bm \Sigma}\nabla f({\bf x})\right].
		$
		\end{theorem*}
		We shall also need the following differentiation formula. For $i\in\left\{1,\cdots,n\right\}$ and $j\in\left\{1,\cdots,p\right\}$
		\begin{equation}
		\frac{\partial \widetilde{\bf Q}_z}{\partial x_{ij}}= -\widetilde{\bf Q}_z{\bf e}_i{\bf e}_j^{T}\frac{\bf PX}{p}\widetilde{\bf Q}_z - \widetilde{\bf Q}_z\frac{{\bf X}^{T}{\bf P}}{p}{\bf e}_j{\bf e}_i^{T}\widetilde{\bf Q}_z,
		\label{eq:differentiation_formula}
		\end{equation}
		where ${\bf e}_i$ is the all zero vector with 1 at the $i$th entry. With this background on the asymptotic behavior of kernel matrices, in the following, we derive the limiting prediction and empirical risks. 
		Recall that the prediction risk writes as
		\begin{align*}
		\mathcal{R}_{\rm test} & =\mathbb{E}_{{\bm s,\epsilon}} \left[\left|{\bm \kappa}_c({\bf s})^{T}\left({\bf K}_c+\lambda {\bf I}_n\right)^{-1}{\bf Py}+\overline{y}-{ f}({\bf s})\right|^2\right],
		\end{align*}
		where ${\bf y}=\left[y_1,\cdots,y_n\right]^{T}={\bf f}({\bf X})+{\bm \epsilon}$ and $\bar{y}=\frac{{\bf 1}^{T}}{n}{\bf f}({\bf X})+\frac{1}{n}{\bm \epsilon}^{T}{\bf 1}$. Due to the independence of ${\bf s}$ and ${\bm \epsilon}$, the prediction risk can be decomposed into a variance and bias terms as
		$
		\mathcal{R}_{\rm test}= B+V,
		$
		where 
		\begin{align*}
		V&=\mathbb{E}_{\bm s,\epsilon}\left[\left|{\bm \kappa}_c({\bf s})^{T}\left({\bf K}_c+\lambda {\bf I}_n\right)^{-1}{\bf P}{\bm \epsilon}+\frac{1}{n}{\bm \epsilon}^{T}{\bf 1}\right|^2\right].\\
		B&=\mathbb{E}_{\bf s}\left[\left({\bm \kappa}_c({\bf s})^{T}\left({\bf K}_c+\lambda {\bf I}_n\right)^{-1}{\bf Pf}({\bf X})+\frac{1}{n}{\bf 1}^{T}{\bf f}({\bf X})-f({\bf s})\right)^2\right].
		\end{align*}
		Now, computing the expectation over ${\bm \epsilon}$, we obtain
		\begin{align*}
		V&=\sigma^2 \mathbb{E}_s\left[{\bm \kappa}_c({\bf s})^{T}\left({\bf K}_c+\lambda{\bf I}_n\right)^{-1}{\bf P}\left({\bf K}_c+\lambda{\bf I}_n\right)^{-1}{\bm \kappa}_c({\bf s})\right] \\ & +\textcolor{black}{O_p(n^{-1})}\\
		&=\sigma^2 \mathbb{E}_s \Biggl[{\bm \kappa}_c({\bf s})^{T}\left({\bf K}_c+\lambda{\bf I}_n\right)^{-2}{\bm \kappa}_c({\bf s}) \\ & -\left|{\bm \kappa}_c({\bf s})^{T}\left({\bf K}_c+\lambda{\bf I}_n\right)^{-1}\frac{{\bf 1}}{\sqrt{n}}\right|^2\Biggr] +\textcolor{black}{O_p(n^{-1})}.
		\end{align*}
		Let us start by controlling the second term. 
		Replacing ${\bm \kappa}_c({\bf s})$ by \eqref{eq:kappa_c},
		\begin{align*}
		& \mathbb{E}_s\left[\left|{\bm \kappa}_c({\bf s})^{T}\left({\bf K}_c+\lambda{\bf I}_n\right)^{-1}\frac{{\bf 1}}{\sqrt{n}}\right|^2\right]  \\ 
		&\leq 3g'(0)^2  \mathbb{E}_{\bm s }\left| \frac{1}{p}{\bf s}^{T}{\bf X}^{T}{\bf P}\left({\bf K}_c+\lambda{\bf I}_n\right)^{-1}\frac{{\bf 1}}{\sqrt{n}}\right|^2\\
		&+ 3\mathbb{E}_{\bm s}\left|\tilde{\bm \kappa}({\bf s})^{T}{\bf P}\left({\bf K}_c+\lambda{\bf I}_n\right)^{-1}\frac{{\bf 1}}{\sqrt{n}}\right|^2+\textcolor{black}{O_p(n^{-1})}.
		\end{align*}
		Computing the expectation over ${\bm s}$ of the term term of the last inequality, we can show that
		\begin{align*}
		&\mathbb{E}_{\bm s}\left[\left|\frac{1}{p}{\bf s}^{T}{\bf X}^{T}{\bf P}\left({\bf K}_c+\lambda{\bf I}_n\right)^{-1}\frac{{\bf 1}}{\sqrt{n}}\right|^2\right] \\ & = \frac{1}{p}\frac{{\bf 1}^{T}}{n}\left({\bf K}_c+\lambda{\bf I}_n\right)^{-1} \frac{{\bf P}{\bf X}{\bm \Sigma}{\bf X}^{T}{\bf P}}{p} \left({\bf K}_c+\lambda{\bf I}_n\right)^{-1}{\bf 1}\\
		&= \textcolor{black}{O_p(n^{-1})}.
		\end{align*}
		On the other hand, from \eqref{eq:kappa_tilde}, we have
		$$
		\mathbb{E}\left|\tilde{\bm \kappa}({\bf s})^{T}{\bf P}\left({\bf K}_c+\lambda{\bf I}_n\right)^{-1}\frac{{\bf 1}}{\sqrt{n}}\right|^2 = \textcolor{black}{O_p(n^{-1})}.
		$$
		The above approximations thus yield
		$$
		\mathbb{E}_s\left[\left|{\bm \kappa}_c({\bf s})^{T}\left({\bf K}_c+\lambda{\bf I}_n\right)^{-1}\frac{{\bf 1}}{\sqrt{n}}\right|^2\right] =O_{p}(n^{-1}).
		$$
		It remains now to compute the first term. Using \eqref{eq:kappa_tilde} along with \eqref{eq:kappa_c}, we have
		\begin{align*}
		V&=\sigma^2\mathbb{E}_{\bm s}\left[\frac{1}{p}(g'(0))^2{\bm s}^{T}{\bf X}^{T}{\bf P}\left({\bf K}_c+\lambda{\bf I}_n\right)^{-2}{\bf P}{\bf X}{\bm s}\right]  +O_p(n^{-1})\\
		&=\sigma^2\frac{1}{p}(g'(0))^2\tr {\bm\Sigma} {\bf X}^{T}{\bf P}\left({\bf K}_c+\lambda{\bf I}_n\right)^{-2}{\bf P}{\bf X} +O_p(n^{-1}).
		\end{align*}
		From \eqref{K_infty}, $V$ can be approximated as
		\begin{align*}
		V\stackrel{(a)}{=}&\sigma^2\frac{1}{p}\tr  {\bm\Sigma} {\bf X}^{T}{\bf P}{\bf Q}_z^2{\bf P}{\bf X} +O_p(n^{-1})\\
		&=\frac{\sigma^2}{p}\frac{\partial}{\partial t} \left[\tr  {\bm\Sigma} {\bf X}^{T}{\bf P}{\bf Q}_t{\bf P}{\bf X}\right]_{t=z} +O_p(n^{-1})\\
		&{=} \sigma^2\frac{\partial}{\partial t} \left[\frac{1}{p}\tr  {\bm\Sigma} {\bf X}^{T}{\bf P}{\bf X}\widetilde{\bf Q}_t\right]_{t=z} +  O_p(n^{-1})\\
		&= \sigma^2\frac{\partial}{\partial t} \left[\frac{1}{p}\tr  {\bm\Sigma} +\frac{t}{p}\tr {\bm \Sigma} \widetilde{\bf Q}_t\right]_{t=z}+  O_p(n^{-1})\\
		&=\frac{\partial}{\partial t} \left[\frac{t\sigma^2}{p}\tr {\bm \Sigma}{\widetilde{\bf Q}_t}\right]_{t=z} +  O_p(n^{-1}),
		\end{align*}
		where in $(a)$ the contribution of matrix   $\frac{\frac{\nu}{g'\left(0\right)} \mathbf{Q}_z \frac{1}{n}\bm{11}^T \mathbf{Q}_z}{1 - \frac{\nu}{g'\left(0\right)} \frac{1}{n} \bm{1}^T \mathbf{Q}_z \bm{1}}$ has been discarded since it only induces terms of order $O_p(n^{-1})$. 
		Using \eqref{eq:trace}, we thus have
		\begin{align*}
		\frac{\partial}{\partial t} \left[\frac{t\sigma^2}{p}\tr {\bm \Sigma}{\widetilde{\bf Q}_t}\right]_{t=z} & = \frac{\partial}{\partial t} \left[-\frac{\sigma^2}{p}\tr {\bm \Sigma} \tr \left({\bf I}_p+m_t{\bm \Sigma}\right)^{-1}\right]_{t=z}  \\ 
		& + O_p(n^{-1}).
		\end{align*}
		Taking the derivative over $z$, we find after simple calculations that
		\begin{align*}
		\frac{\partial}{\partial t} \left[\frac{t\sigma^2}{p}\tr {\bm \Sigma}{\widetilde{\bf Q}_t}\right]_{t=z} & = \frac{\sigma^2m_z^2}{n-m_z^2   \tr {\bm \Sigma}^2({\bf I}_p+m_z{\bm \Sigma})^{-2}} \\ & \times \tr {\bm \Sigma}^2({\bf I}_p+m_z{\bm \Sigma})^{-2} + O_p(n^{-1}).
		\end{align*}
		Putting all the above derivations together, we finally obtain
		
		\begin{equation*}
		V=\frac{\sigma^2m_z^2}{n-m_z^2\tr {\bm \Sigma}^2({\bf I}_p+m_z{\bm \Sigma})^{-2}}\tr {\bm \Sigma}^2({\bf I}_p+m_z{\bm \Sigma})^{-2} + O_p(n^{-1}). 
		\end{equation*}
		
		{\bf Evaluation of the bias term}.
		
		To begin with, we first expand $B$ as
		\begin{align*}
		B&=\mathbb{E}_{\bm s} \left[{\bf f}({\bf X})^{T}{\bf P}\left({\bf K}_c+\lambda {\bf I}_n\right)^{-1}{\bm \kappa}_c({\bf s}){\bm \kappa}_c({\bf s})^{T}\left({\bf K}_c+\lambda {\bf I}_n\right)^{-1}{\bf P}{\bf f}({\bf X})\right] \\ & +\mathbb{E}_{\bm s}\left|\frac{1}{n}{\bf 1}^{T}{\bf f}({\bf X})-f({\bf s})\right|^2
		+2\mathbb{E}_{\bm s} \Biggl[{\bf f}({\bf X})^{T}{\bf P}\left({\bf K}_c+\lambda {\bf I}_n\right)^{-1}{\bm \kappa}_c({\bf s}) \\ & \times \left(\frac{1}{n}{\bf 1}^{T}{\bf f}({\bf X})-f({\bf s})\right) \Biggr].
		\end{align*}
		We will sequentially treat the above three terms. To begin with, we control first $\mathbb{E}_{\bm s}\left[{\bm \kappa}_c({\bm s}){\bm \kappa}_c({\bm s})^{T}\right]$ which we expand as
		\begin{align*}
		&\mathbb{E}_{\bm s}\left[{\bm \kappa}_c({\bm s}){\bm \kappa}_c({\bm s})^{T}\right] \\ & = (g'(0))^2\frac{1}{p^2}{\bf PX}{\bm \Sigma}{\bf X}^{T}{\bf P}  +g'(0)\frac{1}{p}{\bf PX}{\bm s}\tilde{\bm \kappa}(s)^{T}{\bf P}  \\ & +\frac{g'(0)}{p} {\bf P}\mathbb{E}_{\bm s }\left[\tilde{\bm \kappa}({\bm s}){\bf s}^{T}\right]{\bf XP}
		+{\bf P}\mathbb{E}_{\bm s}\tilde{\bm \kappa}({\bf  s})\tilde{\bm \kappa}({\bf  s})^{T}{\bf P}-\frac{1}{n}{\bf P}\mathbb{E}_{\bm s}\tilde{\bm \kappa}({\bm s}){\bf 1}^{T}{\bf KP} \\ & -\frac{1}{n}{\bf PK1}\mathbb{E}_{\bm s}\tilde{\bm \kappa}({\bm s})^{T}{\bf P}+\frac{1}{n^2}{\bf PK11}^{T}{\bf P}.
		\end{align*}
		Using \eqref{eq:kappa_c} along with Lemma \ref{lemma:technical}, we obtain
		\begin{align*}
		\mathbb{E}_{\bm s}\left[{\bm \kappa}_c({\bm s}){\bm \kappa}_c({\bm s})^{T}\right] & =(g'(0))^2\frac{1}{p^2}{\bf PX}{\bm \Sigma}{\bf X}^{T}{\bf P} +\frac{1}{n^2}{\bf PK11}^{T}{\bf KP} \\ & +O_{\|.\|}(n^{-\frac{3}{2}}). 
		\end{align*}
		Replacing ${\bf K}$ by ${\bf K}_{\infty}$, we thus obtain
		\begin{align*}
		\mathbb{E}_{\bm s}\left[{\bm \kappa}_c({\bm s}){\bm \kappa}_c({\bm s})^{T}\right] & =(g'(0))^2\frac{1}{p^2}{\bf PX}{\bm \Sigma}{\bf X}^{T}{\bf P} \\ &  +\frac{(g'(0))^2}{n^2}{\bf P}\frac{{\bf X}{\bf X}^{T}}{p}{\bf 11}^{T}\frac{{\bf XX}^{T}}{p}{\bf P}+O_{\|.\|}(n^{-\frac{3}{2}}).
		\end{align*}
		$$
		$$
		From Assumption \ref{assumption3}, we can prove that
		$$
		\left\|{\bf P}{\bf f}({\bf X})\right\|=O_p(\sqrt{n}).
		$$
		Hence,  the first term in $B$ can be approximated as 
		\begin{align*}
		&\mathbb{E}_{\bm s} \Biggl[{\bf f}({\bf X})^{T}{\bf P}\left({\bf K}_c+\lambda {\bf I}_n\right)^{-1}{\bm \kappa}_c({\bf s}){\bm \kappa}_c({\bf s})^{T}\left({\bf K}_c+\lambda {\bf I}_n\right)^{-1} \\ & \times {\bf P}{\bf f}({\bf X}) \Biggr]\\
		&=\left(g'(0)\right)^2{\bf f}({\bf X})^{T}{\bf P}\left({\bf K}_c+\lambda {\bf I}_n\right)^{-1}{\bf P}\frac{{\bf X}{\bm \Sigma}{\bf X}^{T}}{p^2}{\bf P} \\ & \times\left({\bf K}_c+\lambda {\bf I}_n\right)^{-1} {\bf P}{\bf f}({\bf X}) +\left(g'(0)\right)^2{\bf f}({\bf X})^{T}{\bf P}\left({\bf K}_c+\lambda{\bf I}_n\right)^{-1} \\ & \times \frac{1}{n^2}{\bf P}\frac{{\bf XX}^{T}}{p}{\bf 11}^{T}\frac{{\bf XX}^{T}}{p}{\bf P}({\bf K}_c+\lambda{\bf I}_n)^{-1} {\bf Pf}({\bf X})  \\ & +O_p(n^{-\frac{1}{2}}) \\ 
		&=Z_1+Z_2+O_p(n^{-\frac{1}{2}}).
		\end{align*}
		We will start by treating $Z_1$. From \eqref{K_infty}, we can show that 
		$$
		Z_1=\frac{1}{p^2}{\bf f}({\bf X})^{T}{\bf P}{\bf X}\widetilde{\bf Q}_z{\bm \Sigma}\widetilde{\bf Q}_z{\bf X}^{T} {\bf P}{\bf f }({\bf X})+O_p(n^{-\frac{1}{2}}).
		$$
		To treat $Z_1$, we first decompose ${\bf P}{\bf f}({\bf X})$ as
		$$
		{\bf P}{\bf f}({\bf X}) = {\bf f}({\bf X})-\mathbb{E}_{\bm{x}}(f(\bm{x})){\bf 1} + \mathbb{E}_{\bm{x}}(f(\bm{x})){\bf 1} -\frac{1}{n}{\bf 1}^{T}{\bf f}({\bf X}),
		$$
		where from the classical probability results, we have
		$$
		\left\|\mathbb{E}_{{\bf x}}(f({\bf x})){\bf 1} -\frac{1}{n}{\bf 1}^{T}{\bf f}({\bf X}) {\bf 1}\right\|_2=O_p(1).
		$$
		Hence, we can replace ${\bf P}{\bf f}({\bf X})$ by \\ $\stackrel{\circ}{\bf f}({\bf X})={\bf f}({\bf X})-\mathbb{E}_{\bm{x}}(f(\bm{x})){\bf 1}$ up  to an error $O_p(n^{-\frac{1}{2}})$, thus yielding
		$$
		Z_1=\frac{1}{p^2}\stackrel{\circ}{\bf f}({\bf X})^{T}{\bf X}\widetilde{\bf Q}_z{\bm \Sigma}\widetilde{\bf Q}_z{\bf X}^{T} \stackrel{\circ}{\bf f }({\bf X})+O_p(n^{-\frac{1}{2}}).
		$$
		To treat $Z_1$, we shall resort to the following lemma.
		
		\begin{lemma}
			Let ${\bf A}$ be a $p\times p$ symmetric matrix with a uniformly bounded spectral norm. 
			Let $f$  satisfy Assumption \ref{assumption3} and $z\in\mathbb{C}\backslash\mathbb{R}_{+}$. Consider the following function. 
			\begin{align*}
			{ h}({\bf X}) & =\frac{1}{p^2}\stackrel{\circ}{\bf f}({\bf X})^{T}{\bf X}\widetilde{\bf Q}_z{\bf A}\widetilde{\bf Q}_z{\bf X}^{T} \stackrel{\circ}{\bf f}({\bf X}).
			\end{align*}
			Then, for any $\delta >0$,
			$
			{\rm var}({h}({\bf X}))=O(n^{-1+\delta}).
			$
			\label{lemma:very_technical}
			\end{lemma}
			
			\begin{proof}
				The proof of Lemma \ref{lemma:very_technical} follows from the Nash-Poincar\'e inequality and the differentiation formula in \eqref{eq:differentiation_formula}. Therefore, we omit technical details for brievity. 
			\end{proof}	
				As per Lemma \ref{lemma:very_technical}, $Z_1$ can be approximated as
				$$
				Z_1= \mathbb{E}\left[\frac{1}{p^2}\stackrel{\circ}{\bf f}({\bf X})^{T}{\bf X}\widetilde{\bf Q}_z{\bm \Sigma}\widetilde{\bf Q}_z{\bf X}^{T} \stackrel{\circ}{\bf f }({\bf X})\right]+O_p(n^{-\frac{1}{2}+\delta}),
				$$
				for any $\delta >0$. Now let $\overline{\bf x}=\frac{1}{\sqrt{n}}{\bf X}^{T}{\bf 1}$. The resolvent matrix $\widetilde{\bf Q}_z$ can be expressed as
				$$
				\widetilde{\bf Q}_z=\left(\frac{1}{p}{\bf X}^{T}{\bf X}-\frac{1}{p}\overline{\bf x}\overline{\bf x}^{T}-z{\bf I}_p\right)^{-1}
				.$$
				By the inversion Lemma, matrix $
				\widetilde{\bf Q}_z$ can be written as
				$$
				\widetilde{\bf Q}_z=\overline{\bf Q}_z +\frac{1}{p} \frac{\overline{\bf Q}_z\overline{\bf x}\overline{\bf x}^{T}\overline{\bf Q}_z}{1-\frac{1}{p}\overline{\bf x}^{T}\overline{\bf Q}_z\overline{\bf x}}
				,$$
				where $\overline{\bf Q}_z=\left(\frac{1}{p}{\bf X}^{T}{\bf X}-z{\bf I}_p\right)^{-1}$.
				Hence $Z_1$ can be expanded as
				\begin{align*}
				Z_1&=\frac{1}{p^2}\mathbb{E}\left[\stackrel{\circ}{\bf f}({\bf X})^{T}{\bf X}\overline{\bf Q}_z{\bm \Sigma} \overline{\bf Q}_z{\bf X}^{T}\stackrel{\circ}{\bf f}({\bf X})\right]\\
				&+\frac{1}{p^3}\mathbb{E}\left[\stackrel{\circ}{\bf f}({\bf X})^{T}{\bf X}\overline{\bf Q}_z\overline{\bf x}\overline{\bf x}\overline{\bf Q}_z{\bm \Sigma} \overline{\bf Q}_z{\bf X}^{T}\stackrel{\circ}{\bf f}({\bf X})\left(1-\frac{1}{p}\overline{\bf x}^{T}\overline{\bf Q}_z\overline{\bf x}\right)^{-1}\right]\\
				&+\frac{1}{p^3}\mathbb{E}\left[\stackrel{\circ}{\bf f}({\bf X})^{T}{\bf X}\overline{\bf Q}_z{\bm \Sigma} \overline{\bf Q}_z\overline{\bf x}\overline{\bf x}^{T}{\bf Q}_z{\bf X}^{T}\stackrel{\circ}{\bf f}({\bf X})\left(1-\frac{1}{p}\overline{\bf x}^{T}\overline{\bf Q}_z\overline{\bf x}\right)^{-1}\right]\\
				&+\frac{1}{p^4}\mathbb{E} \Biggl[\stackrel{\circ}{\bf f}({\bf X})^{T}{\bf X}\overline{\bf Q}_z\overline{\bf x}\overline{\bf x}^{T}\overline{\bf Q}_z{\bm \Sigma} \overline{\bf Q}_z\overline{\bf x}\overline{\bf x}^{T}\overline{\bf Q}_z{\bf X}^{T}\stackrel{\circ}{\bf f}({\bf X}) \\ & \times \left(1-\frac{1}{p}\overline{\bf x}^{T}\overline{\bf Q}_z\overline{\bf x}\right)^{-2}\Biggr].
				\end{align*}
				We can show that the last three terms are $O(n^{-\frac{1}{2}+\delta})$. To illustrate this, we will focus on the second term, as the derivations are similar for the remaining quantities. By Cauchy-Schwartz inequality, we have
				\begin{align*}
				&\mathbb{E}\frac{1}{p^3}\left[\stackrel{\circ}{\bf f}({\bf X})^{T}{\bf X}\overline{\bf Q}_z\overline{\bf x}\overline{\bf x}\overline{\bf Q}_z{\bm \Sigma} \overline{\bf Q}_z{\bf X}^{T}\stackrel{\circ}{\bf f}({\bf X})\left(1-\frac{1}{p}\overline{\bf x}^{T}\overline{\bf Q}_z\overline{\bf x}\right)^{-1}\right]\\
				&\leq \sqrt{\mathbb{E}\left|\frac{1}{p^{\frac{3}{2}}}\stackrel{\circ}{\bf f}({\bf X})^{T}{\bf X}\overline{\bf Q}_z{\bf X}^{T}\frac{\bf 1}{\sqrt{n}}\right|^2}\sqrt{\mathbb{E}\left|\frac{1}{p^{\frac{3}{2}}}\frac{{\bf 1}^{T}}{\sqrt{n}}{\bf X}\overline{\bf Q}_z{\bm \Sigma }\overline{\bf Q}_z{\bf X}^{T}\stackrel{\circ}{\bf f}({\bf X})\right|^2}.
				\end{align*}
				To treat the above bound, we first show that
				$$
				\mathbb{E}\left|\frac{1}{p^{\frac{3}{2}}}\frac{{\bf 1}^{T}}{\sqrt{n}}{\bf X}\overline{\bf Q}_z{\bf X}^{T}\stackrel{\circ}{\bf f}({\bf X})\right|^2=O(n^{-1+\delta}),
				$$
				for any $\delta >0$. Towards this end, we need the following control of the variance of  $\frac{1}{p^{\frac{3}{2}}}\frac{{\bf 1}^{T}}{\sqrt{n}}{\bf X}\overline{\bf Q}_z{\bm \Sigma }\overline{\bf Q}_z{\bf X}^{T}\stackrel{\circ}{\bf f}({\bf X})$ which we can be shown to be $O(n^{-1+\delta})$.
				\begin{lemma}
					Let $g$ satisfy Assumption 3 and $z\in\mathbb{C}\backslash\mathbb{R}_{+}$. Consider the following function.
					$$
					{ g}({\bf X}) =\frac{1}{p^{\frac{3}{2}}}\frac{{\bf 1}^{T}}{\sqrt{n}}{\bf X}\overline{\bf Q}_z{\bf X}^{T}\stackrel{\circ}{\bf f}({\bf X}).
					$$
					Then, for any $\delta >0$
					$
					{\rm var}(g({\bf X}))=O(n^{-1+\delta}). 
					$
					\label{lemma:control_variance}
					\end{lemma}
					The proof of Lemma \ref{lemma:control_variance} follows the same lines as Lemma \ref{lemma:very_technical} and is thus omitted. With the control of the variance at hand, it suffices to show the following.
					\begin{equation}
					\mathbb{E}\left[g({\bf X})\right]=O(n^{-\frac{1}{2}+\delta}).
					\label{eq:control_mean}
					\end{equation}
						Let $\overline{\bf Q}_{z,i}=\left(\sum_{k=1, k\notin\left\{i\right\}}^n \frac{1}{p}{\bf x}_k{\bf x}_k^{T}-z{\bf I}_p\right)^{-1}$. 
					We thus develop $g({\bf X})$ as
					\begin{align*}
					&\frac{1}{\sqrt{n}}\frac{1}{p^{\frac{3}{2}}}\sum_{i=1}^n\sum_{j=1}^n \mathbb{E}\left[{\bf x}_i^{T}\overline{\bf Q}_z{\bf x}_j \stackrel{\circ}{\bf f}({\bf x}_j)\right]\\
					&=\frac{1}{\sqrt{n}}\frac{1}{p^{\frac{3}{2}}}\sum_{i=1}^n\sum_{j=1,i\neq j}^n \mathbb{E}\left[ \frac{{\bf x}_i^{T}\overline{\bf Q}_{z,i}{\bf x}_j \stackrel{\circ}{\bf f}({\bf x}_j)}{1+\frac{1}{p}{\bf x}_i^{T}\overline{\bf Q}_{z,i}{\bf x}_i}\right]\\
					&+\frac{1}{\sqrt{n}}\frac{1}{p^{\frac{3}{2}}}\sum_{i=1}^n \mathbb{E}\left[\frac{{\bf x}_i^{T}\overline{\bf Q}_{z,i}{\bf x}_i \stackrel{\circ}{\bf f}({\bf x}_i)}{1+\frac{1}{p}{\bf x}_i^{T}\overline{\bf Q}_{z,i}{\bf x}_i}\right]\\
					&\stackrel{(a)}{=}\frac{1}{\sqrt{n}}\frac{1}{p^{\frac{3}{2}}} \sum_{i=1}^n\sum_{j=1}^n \mathbb{E}{\bf x}_i^{T}\overline{\bf Q}_{z,i}{\bf x}_j \stackrel{\circ}{\bf f}({\bf x}_j) \\ & \times \left[\left(1+\frac{1}{p}{\bf x}_i^{T}\overline{\bf Q}_{z,i}{\bf x}_i\right)^{-1}-\left(1+\frac{1}{p}\tr {\bm \Sigma}\overline{\bf Q}_{z,i}\right)^{-1}\right] \\ & +O(n^{-\frac{1}{2}})\\
					& =  \frac{1}{\sqrt{n}p^{\frac{3}{2}}}\mathbb{E}\stackrel{\circ}{\bf f}({\bf X})^{T}{\bf X}\overline{\bf Q} \\ & \times \diag\left\{\frac{1+\frac{1}{p}{\bf x}_i^{T}\overline{\bf Q}{\bf x}_i}{1+\frac{1}{p}{\bf x}_i^{T}\overline{\bf Q}_{z,i}{\bf x}_i} - \frac{1+\frac{1}{p}{\bf x}_i^{T}\overline{\bf Q}{\bf x}_i}{1+\frac{1}{p}\tr {\bm \Sigma}\overline{\bf Q}_{z,i}}\right\}_{i=1}^n \\ & \times{\bf X}^{T}{\bf 1} +O(n^{-\frac{1}{2}}), 
					\end{align*}
					where in (a), we use the fact that when $i\neq j$, $\mathbb{E}\left[{\bf x}_i^{T}\overline{\bf Q}_{z,i}{\bf x}_j \stackrel{\circ}{\bf f}({\bf x}_j)\left(1+\frac{1}{p}\tr {\bm \Sigma}\overline{\bf Q}_{z,i}\right)^{-1}\right]=0$. Using the fact that
					\begin{align*}
				& 	\left \|\diag\left\{\frac{1+\frac{1}{p}{\bf x}_i^{T}\overline{\bf Q}{\bf x}_i}{1+\frac{1}{p}{\bf x}_i^{T}\overline{\bf Q}_{z,i}{\bf x}_i} - \frac{1+\frac{1}{p}{\bf x}_i^{T}\overline{\bf Q}{\bf x}_i}{1+\frac{1}{p}\tr {\bm \Sigma}\overline{\bf Q}_{z,i}}\right\}_{i=1}^n \right \| \\ & = \textcolor{black}{O_p(n^{-\frac{1}{2}+\delta})}.
					\end{align*}
					we conclude that
					$$
					Z_1= \mathbb{E}\left[\frac{1}{p^2}\stackrel{\circ}{\bf f}({\bf X})^{T}{\bf X}\overline{\bf Q}_z{\bm \Sigma}\overline{\bf Q}_z{\bf X}^{T} \stackrel{\circ}{\bf f }({\bf X})\right]+O_p(n^{-\frac{1}{2}+\delta}).
					$$
					From the inversion lemma, we get 
					$$
					\overline{\bf Q}_z{\bf x}_i=\frac{\overline{\bf Q}_{z,i}{\bf x}_i}{1+\frac{1}{p}{\bf x}_i^{T}\overline{\bf Q}_{z,i}{\bf x}_i}. 
					$$
					$Z_1$ writes as
					\begin{align*}
					Z_1 & = \sum_{i=1}^n\sum_{j=1}^n\frac{1}{p^2}\mathbb{E}\left[\frac{\stackrel{\circ}{f}({\bf x}_i){\bf x}_i^{T}\overline{\bf Q}_{z,i} {\bm \Sigma}\overline{\bf Q}_{z,j}{\bf x}_j \stackrel{\circ}{f}({\bf x}_j)}{(1+\frac{1}{p}{\bf x}_i^{T}\overline{\bf Q}_{z,i}{\bf x}_i)(1+\frac{1}{p}{\bf x}_j^{T}\overline{\bf Q}_{z,j}{\bf x}_j)}\right] \\ & +O_p(n^{-\frac{1}{2}+\delta}).									
					\end{align*}
					Using the fact that
					$$
					\left|\frac{1}{p}{\bf x}_i^{T}\overline{\bf Q}_{z,i}{\bf x}_i+\frac{1}{pz}\tr {\bm \Sigma}({\bf I}_p+m_z{\bm \Sigma})^{-1}\right|^k=O(n^{-\frac{1}{2}}). 
					$$
					Along with the relation
					$$
					1-\frac{1}{pz}\tr {\bm \Sigma}({\bf I}_p+m_z{\bm \Sigma})^{-1}=-zcm_z, 
					$$
					we obtain
					\begin{align}
					Z_1&=c^2z^2 m_z^2 \sum_{i=1}^n\sum_{j=1}^n\frac{1}{p^2}\mathbb{E}\stackrel{\circ}{f}({\bf x}_i){\bf x}_i^{T}\overline{\bf Q}_{z,i} {\bm \Sigma}\overline{\bf Q}_{z,j}{\bf x}_j \stackrel{\circ}{f}({\bf x}_j)\nonumber
					\\
					&=c^2z^2 m_z^2 \sum_{i=1}^n\sum_{j=1,j\neq i}^n\frac{1}{p^2}\mathbb{E}\stackrel{\circ}{f}({\bf x}_i){\bf x}_i^{T}\overline{\bf Q}_{z,i} {\bm \Sigma}\overline{\bf Q}_{z,j}{\bf x}_j \stackrel{\circ}{f}({\bf x}_j)\nonumber\\
					&+c^2z^2m_z^2\frac{1}{p^2}\sum_{i=1}^n \mathbb{E} \left[\left|\stackrel{\circ}{f}({\bf x}_i)\right|^2{\bf x}_i^{T}\overline{\bf Q}_{z,i}{\bm \Sigma}\overline{\bf Q}_{z,i}{\bf x}_i\right]+O_p(n^{-\frac{1}{2}+\delta}). 
					\label{eq:Z1}
					\end{align}
					Now, we further proceed resorting to the inversion lemma
					$$
					\overline{\bf Q}_{z,i}=\overline{\bf Q}_{z,ij} - \frac{1}{p}\frac{\overline{\bf Q}_{z,ij}{\bf x}_j{\bf x}_j^{T}\overline{\bf Q}_{z,ij}}{1+\frac{1}{p}{\bf x}_j^{T}\overline{\bf Q}_{z,ij}{\bf x}_j}, 
					$$
					which once plugged into \eqref{eq:Z1} yields
					\begin{align*}
					Z_1&=c^2z^2 m_z^{2}\sum_{i=1}^n\sum_{j=1}^n \mathbb{E}\left[\frac{1}{p^2} \stackrel{\circ}{f}({\bf x}_i){\bf x}_i^{T}\overline{\bf Q}_{z,ij} {\bm \Sigma}\overline{\bf Q}_{z,ij}{\bf x}_j \stackrel{\circ}{f}({\bf x}_j)\right]\\
					&-c^2z^2 m_z^{2}\sum_{i=1}^n\sum_{j=1,j\neq i}^n \mathbb{E} \Biggl[\frac{1}{p^3}\frac{ \stackrel{\circ}{f}({\bf x}_i){\bf x}_i^{T}\overline{\bf Q}_{z,ij}{{\bf x}_j{\bf x}_j^{T}}\overline{\bf Q}_{z,ij}{\bm \Sigma}\overline{\bf Q}_{z,ij}}{1+\frac{1}{p}{\bf x}_j^{T}\overline{\bf Q}_{z,ij}{\bf x}_j}\\ & \times {\bf x}_j \stackrel{\circ}{f}({\bf x}_j) \Biggr]
					-c^2z^2 m_z^{2}\sum_{i=1}^n\sum_{j=1,j\neq i}^n \mathbb{E} \Biggl[\frac{1}{p^3} \frac{\stackrel{\circ}{f}({\bf x}_i){\bf x}_i^{T}\overline{\bf Q}_{z,ij} {\bm \Sigma}}{1+\frac{1}{p}{\bf x}_i^{T}\overline{\bf Q}_{z,ij}{\bf x}_i}\\ & \times \overline{\bf Q}_{z,ij}{\bf x}_i{\bf x}_i^{T}\overline{\bf Q}_{z,ij}{\bf x}_j \stackrel{\circ}{f}({\bf x}_j) \Biggr] 
	 \\ & +c^2z^2 m_z^{2}\sum_{i=1}^n\sum_{j=1,j\neq i}^n\mathbb{E} \Biggl[\frac{1}{p^4}\frac{ \stackrel{\circ}{f}({\bf x}_i){\bf x}_i^{T}\overline{\bf Q}_{z,ij} {\bf x}_j{\bf x}_j^{T}\overline{\bf Q}_{z,ij} {\bm \Sigma}\overline{\bf Q}_{z,ij}}{(1+\frac{1}{p}{\bf x}_i^{T}\overline{\bf Q}_{z,ij}{\bf x}_i)(1+\frac{1}{p}{\bf x}_j^{T}\overline{\bf Q}_{z,ij}{\bf
					x}_j)} \\ & \times {\bf x}_i{\bf x}_i^{T}\overline{\bf Q}_{z,ij}{\bf x}_j \stackrel{\circ}{f}({\bf x}_j) \Biggr] \\ 
					&\triangleq Z_{11}+Z_{12}+Z_{13}+Z_{14} +O_p(n^{-\frac{1}{2}+\delta}). 
							\end{align*}
							Let us first control $Z_{11}$. Taking the expectation over ${\bf x}_i$ and ${\bf x}_j$, we obtain
							\begin{align*}
							Z_{11}&=c^2 z^2 m_z^2\frac{1}{p^2}\sum_{i=1}^n \sum_{j=1,j\neq i}^n \mathbb{E}\left[\nabla f\right]^{T} \mathbf{\Sigma}\mathbb{E}\left[\overline{\bf Q}_{z,ij}{\bm \Sigma}\overline{\bf Q}_{z,ij}\right] \mathbf{\Sigma} \mathbb{E}\left[\nabla f\right] 
							\\ &+c^2 z^2m_z^2\frac{1}{p^2}\sum_{i=1}^n \mathbb{E}\left[\left|\stackrel{\circ}{f}({\bf x}_i)\right|^2 {\bf x}_i^{T}\overline{\bf Q}_{z,ij}\boldsymbol{\Sigma}\overline{\bf Q}_{z,ij}{\bf x}_i\right]. 
							\end{align*}
							It can be shown that for a vector ${\bf a}$ with uniformly bounded norm that
							$$
							\mathbb{E}\left[{\bf a}^{T}\overline{\bf Q}_{z,ij}{\bm \Sigma}\overline{\bf Q}_{z,ij}{\bf b}\right]= \mathbb{E}\left[{\bf a}^{T}\overline{\bf Q}_{z}{\bm \Sigma}\overline{\bf Q}_{z}{\bf b}\right]+O(n^{-1}). 
							$$
							Using the fact that $\mathbb{E}\left[\left|\frac{1}{p}{\bf x}_i^{T}\overline{\bf Q}_{z,ij}\boldsymbol{\Sigma}\overline{\bf Q}_{z,ij}{\bf x}_i-\mathbb{E}\frac{1}{p}\tr {\bm \Sigma}{\bf Q}{\bm \Sigma}{\bf Q}\right|^2\right]=O(n^{-1})$, we thus obtain
							\begin{align*}
							Z_{11}&= c^2z^2m_z^2\frac{1}{p^2}\sum_{i=1}^n\sum_{j=1,j\neq i}^n \mathbb{E}\left[\nabla f\right]^{T} {\bm \Sigma}\mathbb{E}\left[\overline{\bf Q}_{z}{\bm \Sigma}\overline{\bf Q}_{z}\right]{\bm \Sigma} \mathbb{E}\left[\nabla f\right]\\
							&+c^2z^2m_z^2\frac{1}{p}\sum_{i=1}^n {\bf var}_f \mathbb{E}\left[\frac{1}{p}\tr {\bm \Sigma}{\bf Q}{\bm \Sigma}{\bf Q}\right] +O(n^{-\frac{1}{2}}). 
							\end{align*}
							By standard results from random matrix theory \cite{bilinear,malika}, we get
							\begin{align*}
							Z_{11}&=\frac{nm_z^2 \mathbb{E}\left[\nabla f \right]^{T}{\bm \Sigma}\left({\bf I}_p+m_z{\bm \Sigma}\right)^{-1}{\bm \Sigma}\left({\bf I}_p+m_z{\bm \Sigma}\right)^{-1}}{n-m_z^2\tr {\bm \Sigma}^2\left({\bf I}_p+m_z{\bm \Sigma}\right)^{-2}} \\ & \times {\bm \Sigma}\mathbb{E}\left[\nabla f\right]
							+{\bf var}_f \frac{m_z^2\tr {\bm \Sigma}^2\left({\bf I}_p+m_z {\bm \Sigma}\right)^{-2}}{n-m_z^2\tr {\bm \Sigma}^2\left({\bf I}_p+m_z{\bm \Sigma}\right)^{-2}} \\ & +O(n^{-\frac{1}{2}}). 
							\end{align*}
							Along the same arguments, we can show that
							\begin{align*}
							Z_{12}=Z_{13}& =\frac{m_z^3\tr {\bm \Sigma}^2\left({\bf I}_p+m_z{\bm \Sigma}\right)^{-2}}{n-m_z^2\tr {\bm \Sigma}^2\left({\bf I}_p+m_z{\bm \Sigma}\right)^{-2}} \\ & \times \mathbb{E}\left[\nabla f\right]^T {\bm \Sigma} \left({\bf I}_p+m_z{\bm \Sigma}\right)^{-1}{\bm \Sigma}\mathbb{E}\left[\nabla f\right] +O(n^{-\frac{1}{2}}). 
							\end{align*}
							As for $Z_{14}$, using H\"{o}lder's inequality, it can be bounded as
							\begin{align*}
							& \left|Z_{14}\right| \\  &\leq \frac{c^2z^2m_z^2}{p^2\sqrt{p}} \sum_{i=1}\sum_{j=1,j\neq i} \left(\mathbb{E}\left|\frac{1}{\sqrt{p}}{\bf x}_i^{T}\overline{\bf Q}_{z,ij}{\bf x}_j\right|^4\right)^{\frac{1}{4}} \\ & \times \left(\mathbb{E}\left|\frac{1}{\sqrt{p}}{\bf x}_j^{T}\overline{\bf Q}_{z,ij}{\bm \Sigma}\overline{\bf Q}_{z,ij}{\bf
							x}_i\right|^4\right)^{\frac{1}{4}}  \left(\mathbb{E}\left|\frac{1}{\sqrt{p}}{\bf x}_i^{T}\overline{\bf Q}_{z,ij}{\bf x}_j\right|^4\right)^{\frac{1}{4}} \\ & \times 
							\left(\mathbb{E}\left|\stackrel{\circ}f({\bf x}_i)\stackrel{\circ}f({\bf x}_j)\right|^4\right)^{\frac{1}{4}} \\ 
							& =O(n^{-\frac{1}{2}}) .
								\end{align*}
								We thus conclude that
								\begin{align*}
								Z_1&={\bf var}_f \frac{m_z^2\tr {\bm \Sigma}^2({\bf I}_p+m_z {\bm \Sigma})^{-2}}{n-m_z^2\tr {\bm \Sigma}^2({\bf I}_p+m_z{\bm \Sigma})^{-2}}\\
								&+\frac{nm_z^2 \mathbb{E}\left[\nabla f\right]^{T} {\bm \Sigma}({\bf I}_p+m_z{\bf \Sigma})^{-1}{\bm \Sigma}({\bf I}_p+m_z{\bf \Sigma})^{-1}{\bm \Sigma}\mathbb{E}\left[\nabla f\right]}{n-m_z^2\tr {\bm \Sigma}^2({\bf I}_p+m_z{\bm \Sigma})^{-2}}\\
								&-\frac{2m_z^3\tr {\bm \Sigma}^2({\bf I}_p+m_z{\bm \Sigma})^{-2}}{n-m_z^2\tr {\bm \Sigma}^2({\bf I}_p+m_z{\bm \Sigma})^{-2}} \mathbb{E}\left[\nabla f\right]^{T}{\bm \Sigma} ({\bf I}_p+m_z{\bf \Sigma})^{-1} \\ & \times {\bm \Sigma}\mathbb{E}\left[\nabla f\right]+O_p(n^{-\frac{1}{2}+\delta}). 
								\end{align*}
								Now, we will treat the term $Z_2$. Similarly to $Z_1$, we can show that
								$$
								Z_2= \frac{1}{np^2}{\bf f}({\bf X})^{T}{\bf P}{\bf X}\widetilde{\bf Q}_z {\bf X}^{T} \frac{{\bf 1}{\bf 1}^{T}}{n}{\bf X}\widetilde{\bf Q}_z{\bf X}^{T}{\bf P}{\bf f}({\bf X})+O_p(n^{-\frac{1}{2}}). 
								$$
								Using Lemma \ref{lemma:control_variance} along with \eqref{eq:control_mean}, we thus obtain
								$$
								Z_2=O_p(n^{-\frac{1}{2}}).
								$$
								We have thus completed the treatment of the first term of the bias and shown that
								\begin{align*}
								&\mathbb{E}_{\bm s} \left[{\bf f}({\bf X})^{T}{\bf P}\left({\bf K}_c+\lambda {\bf I}_n\right)^{-1}{\bm \kappa}_c({\bf s}){\bm \kappa}_c({\bf s})^{T}\left({\bf K}_c+\lambda {\bf I}_n\right)^{-1}{\bf P}{\bf f}({\bf X})\right] \\ & ={\bf var}_f \frac{m_z^2\tr {\bm \Sigma}^2({\bf I}_p+m_z {\bm \Sigma})^{-2}}{n-m_z^2\tr {\bm \Sigma}^2({\bf I}_p+m_z{\bm \Sigma})^{-2}}\\
								&+\frac{nm_z^2}{n-m_z^2\tr {\bm \Sigma}^2({\bf I}_p+m_z{\bm \Sigma})^{-2}}  \mathbb{E}\left[\nabla f\right]^{T} {\bm \Sigma}({\bf I}_p+m_z{\bf \Sigma})^{-1} \\ & \times {\bm \Sigma}({\bf I}_p+m_z{\bf \Sigma})^{-1}{\bm \Sigma}\mathbb{E}\left[\nabla f\right] -\frac{2m_z^3\tr {\bm \Sigma}^2({\bf I}_p+m_z{\bm \Sigma})^{-2}}{n-m_z^2\tr {\bm \Sigma}^2({\bf I}_p+m_z{\bm \Sigma})^{-2}} \\ & \times \mathbb{E}\left[\nabla f\right]^{T}{\bm \Sigma} ({\bf I}_p+m_z{\bf \Sigma})^{-1}{\bm \Sigma}\mathbb{E}\left[\nabla f\right]+O_p(n^{-\frac{1}{2}+\delta}). 
								\end{align*}
								The second term in the bias can be dealt with by noticing that
								$$
								\frac{1}{n}{\bf 1}^{T}f({\bf X}) = \mathbb{E}f({\bf x}) +O_p(\frac{1}{\sqrt{n}}). 
								$$
								Thus yielding
								$$
								\mathbb{E}_{\bm s}\left|\frac{1}{n}{\bf 1}^{T}{\bf f}({\bf X})-f({\bf s})\right|^2={\bf var}_f+O_p(n^{-\frac{1}{2}}). 
								$$
								We now move to the last term in the bias. Using \eqref{eq:kappa_c}, we obtain
								\begin{align*}
								& 2\mathbb{E}_{\bm s}\left[{\bf f}({\bf X})^{T}{\bf P}\left({\bf K}_c+\lambda {\bf I}_n\right)^{-1}\tilde{\bm\kappa}_c({\bm s})\left(\frac{1}{n}{\bf 1}^{T}{\bf f}({\bf X})-f({\bm s})\right)\right]\\
								& =2\mathbb{E}_{\bm s} \Biggl[ {\bf f}({\bf X})^{T}{\bf P}\left({\bf K}_c+\lambda {\bf I}_n\right)^{-1}\left\{g'(0)\frac{1}{p}{\bf PX}{\bm s}+{\bf P}\tilde{\bm \kappa}_c({\bm s})-\frac{1}{n}{\bf PK}{\bf 1}\right\} \\ & \times \left(\frac{1}{n}{\bf 1}^{T}{\bf f}({\bf X})-f({\bm s})\right) \Biggr]\\
								&=-2\mathbb{E}_{\bm s}\left[{\bf f}({\bf X})^{T}{\bf P}\left({\bf K}_c+\lambda {\bf I}_n\right)^{-1}g'(0)\frac{1}{p}{\bf PX}{\bm s}f({\bm s})\right]\\
								&+2\mathbb{E}_{\bm s}\left[{\bf f}({\bf X})^{T}{\bf P}\left({\bf K}_c+\lambda {\bf I}_n\right)^{-1}{\bf P}\tilde{\bm \kappa}_c({\bm s}){\bm s}\left(\frac{1}{n}{\bf 1}^{T}{\bf f}({\bf X})-f({\bm s})\right)\right]\\
								&-2\mathbb{E}_{\bm s}\left[{\bf f}({\bf X})^{T}{\bf P}\left({\bf K}_c+\lambda {\bf I}_n\right)^{-1}\frac{1}{n}{\bf PK}{\bf 1}\left(\frac{1}{n}{\bf 1}^{T}{\bf f}({\bf X})-f({\bm s})\right)\right]. 
								\end{align*}
								Since $\mathbb{E}_{\bm s}f({\bm s}) -\frac{1}{n}{\bf 1}^{T}{\bf f}({\bf X})=O_p(n^{-\frac{1}{2}})$, we can replace in the two last terms $\frac{1}{n}{\bf 1}^{T}{\bf f}({\bf X})$ by $\mathbb{E}_{\bm s}f({\bm s})$ with an error $O(n^{-\frac{1}{2}})$. In doing so, we obtain
								\begin{align*}
								& 2\mathbb{E}_{\bm s}\left[{\bf f}({\bf X})^{T}{\bf P}\left({\bf K}_c+\lambda {\bf I}_n\right)^{-1}\tilde{\bm\kappa}_c({\bm s})\left(\frac{1}{n}{\bf 1}^{T}{\bf f}({\bf X})-f({\bm s})\right)\right]\\
								&=-2{\bf f}({\bf X})^{T}{\bf P}\left({\bf K}_c+\lambda {\bf I}_n\right)^{-1}g'(0)\frac{1}{p}{\bf PX}{\bm \Sigma}\mathbb{E}\left[\nabla f\right]\\
								&+2\mathbb{E}_{\bm s}\left[{\bf f}({\bf X})^{T}{\bf P}\left({\bf K}_c+\lambda {\bf I}_n\right)^{-1}{\bf P}\tilde{\bm \kappa}_c({\bm s})\left(\mathbb{E}_{\bm s}f({\bf s})-f({\bm s})\right)\right] \\ & +O_p(n^{-\frac{1}{2}}). 
								\end{align*}
								Using \eqref{K_infty} along with standard calculations as those used in Lemma \ref{lemma:control_variance}, we obtain
								\begin{align*}
								&-2{\bf f}({\bf X})^{T}{\bf P}\left({\bf K}_c+\lambda {\bf I}_n\right)^{-1}g'(0)\frac{1}{p}{\bf PX}{\bm \Sigma}\mathbb{E}\left[\nabla f\right] \\ & = -2{\bf f}({\bf X})^{T}{\bf P}{\bf Q}_z\frac{1}{p}{\bf PX}{\bm \Sigma}\mathbb{E}\left[\nabla f\right]+O_p(n^{-\frac{1}{2}}) \\
								& =-\frac{2}{p}{\bf f}({\bf X})^{T}{\bf P}{\bf X}\widetilde{\bf Q}_z{\bm \Sigma}\mathbb{E}\nabla f +O_p(n^{-\frac{1}{2}}).
								\end{align*}
								Replacing $\widetilde{\bf Q}_z$ by $\overline{\bf Q}_z$, and using similar derivations as before, we obtain
								\begin{align*}
								& -2{\bf f}({\bf X})^{T}{\bf P}\left({\bf K}_c+\lambda {\bf I}_n\right)^{-1}g'(0)\frac{1}{p}{\bf PX}{\bm \Sigma}\mathbb{E}\left[\nabla f\right]  \\  &=-2m_z \mathbb{E}\left[\nabla f\right]^{T}{\bm \Sigma}\left({\bf I}_p+m_z{\bm \Sigma}\right)^{-1}{\bm \Sigma}\mathbb{E}\left[\nabla f\right]   +O_p(n^{-\frac{1}{2}+\delta}).
								\end{align*}
								Finally, using the same tools we can show that 
								$$2\mathbb{E}_{\bm s}\left[{\bf f}({\bf X})^{T}{\bf P}\left({\bf K}_c+\lambda {\bf I}_n\right)^{-1}{\bf P}\tilde{\bm \kappa}_c({\bm s})\left(\mathbb{E}_{\bm s}f({\bf s})-f({\bm s})\right)\right]$$ will not contribute to the expression of the bias. 
In fact, 						
								\begin{align*}
							& 	\frac{2}{p^2}\mathbb{E}_{\bm s}\left[{\bf f}({\bf X})^{T}{\bf P}\left({\bf K}_c+\lambda {\bf I}_n\right)^{-1}{\bf P}\tilde{\bm \kappa}_c({\bm s})\left(\mathbb{E}_{\bm s}f({\bf s})-f({\bm s})\right)\right] \\ & =O_p(n^{-\frac{1}{2}+\delta}). 
								\end{align*}
								We are now in position to estimate the bias term. Combining the results of all derivations, for any $\delta >0$ we obtain
								\begin{align*}
								B&=\frac{n {\bf var}_f}{n-m_z^2\tr \tr {\bm \Sigma}^2({\bf I}_p+m_z{\bm \Sigma})^{-2}}\\
								& -\frac{nm_z}{n-m_z^2\tr \tr {\bm \Sigma}^2({\bf I}_p+m_z{\bm \Sigma})^{-2}}\mathbb{E}\left[\nabla f\right]^{T}{\bm \Sigma}\left({\bf I}_p+m_z{\bm \Sigma}\right)^{-1} \\ & \times {\bm \Sigma}\mathbb{E}\left[\nabla f\right]
								-\frac{nm_z^2}{n-m_z^2\tr \tr {\bm \Sigma}^2({\bf I}_p+m_z{\bm \Sigma})^{-2}}\mathbb{E}\left[\nabla f\right]^{T}{\bm \Sigma} \\ & \times \left({\bf I}_p+m_z{\bm \Sigma}\right)^{-2}  \times {\bm \Sigma} \mathbb{E}\left[\nabla f\right]+O_p(n^{-\frac{1}{2}+\delta}). 
								\end{align*}
						This concludes the proof. 		
\section*{Appendix B}
\section*{Proof of Theorem \ref{consistent_estim} }		
The proof heavily relies on the same tools used earlier in the proof of Theorem \ref{theorem1} and on the following observations 
\begin{align*}
& \frac{1}{p^2} \bm{y}^T \mathbf{PX} \mathbf{\widetilde{Q}}_z \mathbf{X}^T \mathbf{P} \bm{y}  \\ & = \frac{1}{p^2} f\left(\mathbf{X}\right)^T \mathbf{PX} \mathbf{\widetilde{Q}}_z \mathbf{X}^T \mathbf{P} f\left(\mathbf{X}\right) + 	\frac{1}{p^2} \bm{\epsilon}^T \mathbf{PX} \mathbf{\widetilde{Q}}_z \mathbf{X}^T \mathbf{P} \bm{\epsilon}   \\ & + O_p\left(n^{-\frac{1}{2}}\right) \\ 
&= \frac{1}{p^2} f\left(\mathbf{X}\right)^T \mathbf{PX} \mathbf{\widetilde{Q}}_z \mathbf{X}^T \mathbf{P} f\left(\mathbf{X}\right) + \frac{\sigma^2}{p} \tr \frac{\mathbf{X}^T \mathbf{PX}}{p} \mathbf{\widetilde{Q}}_z \\ & +  O_p\left(n^{-\frac{1}{2}}\right)  \\ 
& =  \frac{1}{p^2} f\left(\mathbf{X}\right)^T \mathbf{PX} \mathbf{\widetilde{Q}}_z \mathbf{X}^T \mathbf{P} f\left(\mathbf{X}\right) + \sigma^2 + z \frac{\sigma^2}{p} \tr \mathbf{\widetilde{Q}}_z  \\ & +  O_p\left(n^{-\frac{1}{2}}\right) \\ 
& = \frac{1}{p^2} f\left(\mathbf{X}\right)^T \mathbf{PX} \mathbf{\widetilde{Q}}_z \mathbf{X}^T \mathbf{P} f\left(\mathbf{X}\right) + \sigma^2 - \frac{\sigma^2}{p} \tr \mathbf{\Sigma }\left(\mathbf{I}+ m_z\mathbf{\Sigma}\right)^{-1}  \\ & + O_p\left(n^{-\frac{1}{2}}\right) .
\end{align*}
It is also possible to show along the same lines as before that
\begin{align*}
& \frac{1}{p^2} f\left(\mathbf{X}\right)^T \mathbf{PX} \mathbf{\widetilde{Q}}_z \mathbf{X}^T \mathbf{P} f\left(\mathbf{X}\right) \\  &  = -z m_z^2 \E\left[\nabla_f\right]^T \mathbf{\Sigma}\left(\mathbf{I}+ m_z\mathbf{\Sigma}\right)^{-1} \mathbf{\Sigma} \E\left[\nabla_f\right]  \\ & + m_z \textbf{var}_f\frac{1}{p} \tr \mathbf{\Sigma }\left(\mathbf{I}+ m_z\mathbf{\Sigma}\right)^{-1}  + O_p\left(n^{-\frac{1}{2}+\delta}\right).
\end{align*}
Moreover, by computing 
\begin{align*}
& \frac{1}{p^2} f\left(\mathbf{X}\right)^T \mathbf{PX} \mathbf{\widetilde{Q}}^2_z \mathbf{X}^T \mathbf{P} f\left(\mathbf{X}\right) \\ & = \frac{\partial }{\partial z} \frac{1}{p^2} f\left(\mathbf{X}\right)^T \mathbf{PX} \mathbf{\widetilde{Q}}_z \mathbf{X}^T \mathbf{P} f\left(\mathbf{X}\right) \\
& =  \frac{\partial }{\partial z} \Biggl[-z m_z^2 \E\left[\nabla_f\right]^T \mathbf{\Sigma}\left(\mathbf{I}+ m_z\mathbf{\Sigma}\right)^{-1} \mathbf{\Sigma} \E\left[\nabla_f\right]  \\ & + m_z \textbf{var}_f\frac{1}{p} \tr \mathbf{\Sigma }\left(\mathbf{I}+ m_z\mathbf{\Sigma}\right)^{-1}\Biggr] 
+ O_p\left(n^{-\frac{1}{2}+\delta}\right). 
\end{align*}
With the above observation at hand, it is straightforward to show that 
\begin{align*}
\widehat{\mathcal{R}}_{test} = \mathcal{R}_{test}^{\infty} + O_p(n^{-\frac{1}{2}+\delta}),
\end{align*}
which is combined with Theorem \ref{theorem1} gives the claim of Theorem \ref{consistent_estim}.				
\bibliographystyle{IEEEtran}
\bibliography{References_zhang}



\end{document}